\documentclass{article}

\usepackage{microtype}
\usepackage{graphicx}
\usepackage{subfigure}
\usepackage{booktabs} % for professional tables

\usepackage[backref=page]{hyperref}

\usepackage[accepted]{icml2023}

\usepackage{amsmath}
\usepackage{amssymb}
\usepackage{mathtools}
\usepackage{amsthm}

\usepackage[capitalize,noabbrev]{cleveref}

\usepackage{titletoc}
\usepackage{wustyle}

\theoremstyle{plain}
\newtheorem{theorem}{Theorem}[section]
\newtheorem{proposition}[theorem]{Proposition}
\newtheorem{lemma}[theorem]{Lemma}

\theoremstyle{definition}
\newtheorem{definition}[theorem]{Definition}
\newtheorem{assumption}[theorem]{Assumption}
\theoremstyle{remark}
\newtheorem{remark}[theorem]{Remark}

\usepackage[textsize=tiny]{todonotes}

\icmltitlerunning{The Implicit Regularization of Dynamical Stability in Stochastic Gradient Descent}

\begin{document}

% \twocolumn[
% \icmltitle{The implicit regularization of  dynamical stability in \\ stochastic gradient descent}

\twocolumn[
% \icmltitle{The regularization effect of dynamical dtability in stochastic gradient descent}
% \icmltitle{Dynamical Stability and Implicit Regularization in Stochastic Gradient Descent}
\icmltitle{The Implicit Regularization of Dynamical Stability in Stochastic\\ Gradient Descent}



\icmlsetsymbol{equal}{*}

\begin{icmlauthorlist}
\icmlauthor{Lei Wu}{pku,cmlr}
\icmlauthor{Weijie J. Su}{uppen}
% \icmlauthor{Firstname3 Lastname3}{comp}
\end{icmlauthorlist}

\icmlaffiliation{pku}{School of Mathematical Sciences, Peking University, Beijing, China}
\icmlaffiliation{cmlr}{Center for Machine Learning Research, Peking University, Beijing, China}
\icmlaffiliation{uppen}{Wharton Statistics and Data Science Department, University of Pennsylvania,  Philadelphia, USA}

\icmlcorrespondingauthor{Lei Wu}{leiwu@math.pku.edu.cn}
\icmlcorrespondingauthor{Weijie J. Su}{suw@wharton.upenn.edu}

\icmlkeywords{Implicit regularization, dynamical stability, Stochastic gradient descent}

\vskip 0.3in
]



\printAffiliationsAndNotice{}  % leave blank if no need to mention equal contribution

\begin{abstract}
In this paper, we study the implicit regularization of stochastic gradient descent (SGD) through the lens of {\em dynamical stability} \citep{wu2018sgd}. We start by revising existing stability analyses of SGD, showing how the Frobenius norm and trace of Hessian relate to  different notions of stability. 
Notably, if a global minimum is linearly stable for SGD, then the trace of Hessian must be less than or equal to $2/\eta$, where $\eta$ denotes the learning rate. 
By contrast, for gradient descent (GD), the stability imposes a similar constraint but only on  the largest eigenvalue of Hessian. We then turn to analyze the generalization properties of these stable minima, focusing specifically on   two-layer ReLU networks and diagonal linear networks. Notably, we establish the {\em equivalence} between these metrics of sharpness and certain parameter norms for the two models, which allows us to  show that  the stable minima of SGD provably generalize well. By contrast, the stability-induced regularization of GD  is provably too weak to ensure satisfactory generalization. This discrepancy  provides an explanation of why SGD often generalizes better than GD.
Note that the learning rate (LR) plays a pivotal role in the strength of stability-induced regularization. As the LR increases, the regularization effect becomes more pronounced, elucidating why SGD with a larger LR consistently demonstrates superior generalization capabilities. Additionally, numerical experiments are provided to support our theoretical findings.
\end{abstract}

\section{Introduction}
In modern machine learning, models are often  over-parameterized in the sense that they can easily interpolate all training data. Therefore, one may be concerned that algorithms may pick up  solutions that generalize badly on test data \cite{wu2017towards}. Fortunately, it has been  found that  simple SGD and its variants always converge to solutions that generalize well, even without empolying any explicit regularizations \cite{zhang2017understanding}. Furthermore,  SGD often  generalizes better than GD \cite{keskar2016large}.  Hence, there must exist certain ``implicit regularization'' mechanisms at work  \cite{neyshabur2014search}. As practitioners increasingly rely on implicit regularization to mitigate overfitting, it becomes imperative to understand the underlying mechanisms.

The most popular explanation is the {\em flat-minima hypothesis}: SGD tends to select flat minima \cite{keskar2016large} and flat minima generalize well \cite{hochreiter1994simplifying,hochreiter1997flat}. 
This  hypothesis has been widely  adopted in practice to tune the hyperparameters of SGD \cite{keskar2016large,jastrzkebski2017three,wu2020noisy} and to design new optimizers \cite{izmailov2018averaging,foret2020sharpness,wu2020adversarial} for better generalization. Despite its widespread use, the theoretical understanding  is still largely lacking: 1) Why does SGD favor flat minima? 2) Why do flat minima generalize?

% In this paper, we provide theoretical answers to the two questions by following the dynamical stability perspective proposed in 
In this paper, we aim to address these questions by adopting the perspective of dynamical stability \cite{wu2018sgd,wu2022alignment}. For over-parameterized models, all global minima are fixed points of SGD but their  stability can be different. Notably, when confronted with a small perturbation, SGD steers away from unstable minima, while stable minima tend to be more resilient, allowing SGD to persist and even reconverge after initial perturbations. This intriguing behavior suggests that SGD exhibits a preference for stable minima. The remaining puzzle lies in understanding the relationship between the stability of a minimum, its sharpness, and its generalization properties.
% To ease the understanding, one can compare saddle points and  minima for gradient flow (GF). Although they are both fixed points of GF, but GF amost surely converge to  minima \citep{lee2016gradient}  since saddle points are dynamically unstable.

It is well-known that the stability condition for GD is  $\|H(\theta)\|_2\leq 2/\eta$ \cite{wu2018sgd}, where $H(\cdot)$ denotes the Hessian matrix. This implies that  GD tends to select minima whose sharpness, as measured by the spectral norm of Hessian, is bounded independently of the model size and sample size. 
\citet{mulayoff2021implicit,nacson2022implicit} showed that for univariate two-layer ReLU networks and diagonal linear networks, this sharpness can control the model capacity under some data assumption. Therefore, the stability ensures that GD selects flat minima that generalize well  for these models.

Then a natural question is: {\em Can we establish a similar understanding of SGD?} \citet{ma2021linear} showed that if a global minimum is {\em linearly stable} for SGD, then the trace of Hessian  $\tr(H(\theta))$ must be bounded. Meanwhile, it was also proved that for  ReLU networks, $\tr(H(\theta))$ can control the Sobolev seminorm of the functions implemented. These  together provide insight into  how dynamical stability can act as a form of regularization in SGD. See \citet{wu2017towards} for a similar argument. However, it is important to note that this smoothness-based generalization cannot explain the superiority of neural networks in high dimensions \citep{barron1993universal} as the resulting generalization error bound suffers from the curse of dimensionality. 
The major reason is that the upper bound of $\tr(H(\theta))$ obtained in \citet{ma2021linear} grows linearly with the number of parameters. 
 In contrast,  by introducing a new notion of stability, \citet{wu2022alignment} showed that the stability imposes a size-independent control on the Frobenius norm of Hessian: $\|H(\theta)\|_F$ but \citet{wu2022alignment} did not discuss the corresponding generalization properties. In a word, understanding the stability-induced regularization is still incomplete for SGD and in particular, the following critical questions remain to be answered:
\begin{itemize}
    \vspace*{-.5em}
\item Can we show that the stable minima of SGD generalize well in high dimensions?
\item Can we explain why SGD generalizes better than GD?
\end{itemize}

\paragraph*{Our contributions.} 

We begin by presenting an improved stability analysis for SGD, demonstrating that stability imposes a size-independent control on either the Frobenius norm or the trace of the Hessian matrix, depending on the notion of stability used. 
Specifically, if a global minimum $\theta$ is {\em linearly stable} for SGD, then $\tr(H(\theta))\leq 2/\eta$; if it satisfies a loss stability, then $\|H(\theta)\|_F= O(1/\eta)$. In contrast, the stability of GD only controls the largest eigenvalue of Hessian: $\|H(\theta)\|_2\leq 2/\eta$.
We then examine the implications of  these stability conditions  for generalization, and  our main findings are summarized as follows.
\begin{itemize} 
    \vspace*{-.2em}
\item We first consider two-layer ReLU networks. It is proved that all  three  aformentioned measures of sharpness can effectively bound the path norm \cite{neyshabur2015norm,ma2019priori}, thus controlling the generalization gap. As a result, for both SGD and GD, the stable minima are guaranteed to generalize well, which stems from the stability conditions that impose constraints ensuring that the path norms remain bounded by $O(1/\eta)$, irrespective of the model's size. Thus, the size-independent nature of sharpness control strengthens the assurance of favorable generalization properties for the stable minima.
% In contrast to \citep{ma2021linear,mulayoff2021implicit}, our generalization bound is effective in high dimensions. 

\item We next delve into the analysis of diagonal linear networks, which are essentially over-parameterized linear models. We prove that the spectral norm, Frobenius norm, and trace of Hessian are roughly equivalent to the $\ell_\infty$, $\ell_2$ and $\ell_1$ norm of the effective coefficients, respectively. 
The stability of GD only guarantees a size-independent control on the spectral norm of Hessian, thereby the $\ell_\infty$ norm of effective coefficients. 
 Consequently, stable minima of GD may not generalize well since the $\ell_\infty$ norm cannot yield an effective capacity control for linear models. In stark contrast, the stability of SGD imposes  size-independent controls on the trace or Frobenius norm of Hessian, thereby  the $\ell_1$ or $\ell_2$ norm of effective coefficients. As a result, the stable minima of SGD must generalize well. 

 This comparison between SGD and GD effectively demonstrates that in the case of diagonal linear networks, the stability of SGD imparts a substantially stronger regularization effect than that of GD. This provides an explanation for the superior generalization performance consistently observed in SGD over GD.
\end{itemize}

It is important to note that the strength of stability-induced regularization  crucially depends on the size of LR. A larger LR imposes a stricter constraint on the sharpness of stable minima, thereby enforcing SGD/GD to select flatter minima. This explains why SGD with a large LR often generalizes better.
To support our theoretical findings, systematic numerical experiments are provided  and in particular, we examine in detail the impact of varying LRs.
% We remark that despite the simplicity, these two models have been widely used  in theoretical analysis to  explain the particular properties of neural networks.

% \vspace*{-.3em}
\subsection{Related works}
\vspace*{-.4em}
\paragraph*{Implicit regularization of SGD.}
  
In SGD, there exist multiple mechanisms that contribute to the implicit regularization \cite{su2021neurashed,he2019local,vardi2022implicit}.  One is the specific dynamical process, along with small initialization, aiding SGD in finding solutions that generalize well \cite{zhang2017understanding,woodworth2020kernel,blanc2020implicit,chizat2020implicit,pesme2021implicit,ma2020quenching,xu2021towards}. This type of implicit regularization heavily relies on the initialization size. In contrast,   the stability-induced regularization \cite{wu2018sgd} is independent of the initialization and can explain why using a large LR and small batch size is more favorable \cite{wu2022alignment,ma2021linear}. In the experiments of the current work,  
we intentionally exclude the small initialization-induced regularization by using a large initialization, as our focus is understanding the stability-induced regularization.

\citet{barrett2020implicit,smith2020origin} explained the benefit of using a large LR through a modified equation analysis. However, the analysis is  only  validated for a finite time and hence, cannot explain why SGD favors certain minima,  as the latter is a long-time property. In contrast, our stability analysis does not have this limitation.

\vspace*{-.4em}
\paragraph*{Generalization of flat minima.} To explain why flat minima generalize well, many works rely on the PAC-Bayesian  argument \citep{mcallester1999some}. This argument established the connection between  certain sharpness and the average generalization error of perturbed solutions, which, however, is not for the original one \citep{neyshabur2017,tsuzuku2020normalized}. Furthermore, Bayesian arguments  tend to ignore the specific parametrization of neural networks. \cite{mulayoff2021implicit,ma2021linear,nacson2022implicit} established sharpness-based generalization bounds for some neural networks but they are limited to either linear or low-dimensional  cases. In contrast, our sharpness-based generalization bound of two-layer ReLU networks is effective in high dimensions.

In addition, some works argue that sharpness itself can not effectively control model capacity since  ReLU neural nets are invariant to node-wise rescaling, whereas the sharpness is not. Consequently,  sharp minima can generalize well  \citep{dinh2017sharp}. To overcome this issue, various rescaling-invariant sharpness has been proposed, e.g., the Fisher-Rao metric \citep{pmlr-v89-liang19a}, normalized flatness \citep{tsuzuku2020normalized}, relative flatness \citep{petzka2021relative}. However, our analysis suggests that flatness can be {\em sufficient} for good generalization. 
% We also argue that pursuing a  notion of rescaling-invariant sharpness might be irrelevant to understanding the implicit regularization of SGD since SGD iterations are not rescaling-invariant. 

\vspace*{-.5em}
\paragraph*{Notation.}
For an integer $k$, let $[k]=\{1,2,\dots,k\}$.
For a vector $v$, let $\|v\|_{p}=(\sum_i v_i^p)^{1/p}$, $\|\cdot\|=\|\cdot\|_2$, and  $\hv=v/\|v\|_2$.
For a matrix $A$, denote by $\|A\|_2$ and $\|A\|_F$ the spectral norm and the Frobenius norm, respectively and let $\{\lambda_i(A)\}_{i\geq 1}$ be the eigenvalues of $A$ in a decreasing order. Let $\SS^{d-1}=\{x\in\RR^d\,|\,\|x\|=1\}$ and $r\SS^{d-1}=\{x\in \RR^d \,|\, \|x\|=r\}$. For a distribution $\mu$, let $\|f\|^2_{L_2(\mu)}=\EE_{x\sim\mu}[f^2(x)]$.
We will use $C$ to denote an absolute constant, whose value may change from line to line. For notation simplicity, we write $X\lesssim Y$ if $X\leq CY$ and  $X\gtrsim Y$ if $X\geq CY$. Analogously, we write $X\sim Y$ if  $X\lesssim Y$ and $X\gtrsim Y$ hold simultaneously. 

\vspace*{-.2em}
\section{Preliminaries}
\vspace*{-.1em}
\label{sec: pre}

Let $S=\{(x_i, y_i=f^*(x_i))\}_{i=1}^n$  be the training set, where $x_1,\dots,x_n$ are \iid samples drawn from the input distribution $\rho$ and $f^*:\RR^d\mapsto\RR$ be the target function. Our task is to recover $f^*$ from $S$.
Let $f(\cdot;\theta):\RR^d\mapsto\RR$ be our model parameterized by $\theta\in\RR^p$, where $d$ and $p$ denote the input dimension and the model size (i.e., the number of parameters), respectively. 
The empirical and population risk are  given by 
\vspace*{-1em}
\begin{equation}
\begin{aligned}
\erisk(\theta) &= \frac{1}{2n} \sumin (f(x_i;\theta)-y_i)^2\\
\risk(\theta) &= \half\EE_{x,y}[(f(x;\theta)-y)^2],
\end{aligned}
\vspace*{-1em}
\end{equation}

where the square loss is used.
Throughout this paper, we make the following   over-parameterization assumption.
\begin{assumption}[Over-parameterization]
$\min_{\theta} \erisk(\theta)=0$
\end{assumption}
Let $g_i(\theta)=\nabla f(x_i;\theta)$ and $e_i(\theta) = f(x_i;\theta)-y_i$. Then the Hessian matrix is given by 

\vspace*{-1.2em}
\begin{equation}\label{eqn: hessian-expression}
  H(\theta) = \fn \sumin g_i(\theta) g_i(\theta)^T + \fn \sumin e_i(\theta) \nabla^2 f(x_i;\theta).
 % \vspace*{-1em}
\end{equation}
Let $G(\theta)=\fn \sumin g_i(\theta) g_i(\theta)^T$ be the associate empirical Fisher matrix. 
Then, \eqref{eqn: hessian-expression} implies that when fitting errors are small, we have $H(\theta)\approx G(\theta)$ and in particular,  $H(\theta)=G(\theta)$ if $\theta$ is a global minimum. 
% In addition, let $H(\theta)$ and $G(\theta)$ be the population Hessian and Fisher matrix, respectively.
Note that $G(\theta)$ is always positive semi-definite but $H(\theta)$ is not. In our analysis of the dynamical stability, we shall focus on the region with small empirical risk and hence,  we do not distinguish the Fisher matrix and Hessian matrix too much since they are close to each other.

\paragraph*{Gradient clipping.} In our experiments, we will use large initialization to exclude the implicit regularization induced by small initialization. This choice will make it very often that SGD and GD with a large LR 
diverge initially although there exist stable minima on landscape. To resolve this issue, we shall apply gradient clipping \cite{pascanu2013difficulty,mikolov2012statistical} to stabilize the training. Specifically,  we use the following clipped (stochastic) gradient for SGD/GD update:
\vspace*{-.5em}
\[
    \nabla \erisk_{clip}(\theta) = \min\{\|\nabla \erisk(\theta)\|, \delta \} \frac{\nabla \erisk(\theta)}{\|\nabla \erisk(\theta)\|},
\vspace*{-.2em}
\]
where $\delta$ denotes the clipping threshold. In all our experiments, we find that gradient clipping is activated only during the early and intermediate training stages, and will be  automatically switched off when SGD/GD nearly converges since the gradient norm there is lower than the clipping threshold. Therefore, gradient clipping does not change the dynamical stability of SGD/GD at global minima.

\vspace*{-.2em}
\section{The dynamical stability of SGD}

\label{sec: stability}
In this section, we  consider three measures of sharpness: $\|G(\theta)\|_2$, $\|G(\theta)\|_F$, and  $\tr(G(\theta))$ and study how they are related to the stability of  SGD and GD

It is well-known that  if $\theta$ is a linearly stable for GD, then
$
\|H(\theta)\|_2\leq 2/\eta 
$ \citep{wu2018sgd,mulayoff2021implicit}, which implies  $\|G(\theta)\|_2\leq 2/\eta$ if  $\theta$ is a global minima. Next, we will show that similar size-independent controls  hold for SGD but on different norms of Fisher matrix.

% \paragraph*{Stochastic }

\vspace*{-.2em}
\subsection{Linear stability} 
\label{sec: linear-stability}

Consider the mini-batch SGD:
\begin{equation}\label{eqn: mini-batch-sgd}
\theta_{t+1} = \theta_t - \eta (f(x_{i_t};\theta_t)-y_{i_t})\nabla f(x_{i_t};\theta_t),
\end{equation}
where $i_t\stackrel{iid}{\sim} \unif([n])$. Throughout this paper, we assume the batch size to be $1$ 
 for simplicity.

Suppose that  $\theta_t$ converges to a global minimum $\theta^*$. Let $\delta_t=\theta_t-\theta^*$ be the deviation. When $\|\delta_t\|$ is small,  $f(x;\theta_t)=f(x;\theta^*)+\nabla f(x;\theta^*)^T\delta_t + o(\|\delta_t\|)$. Substituting it into \eqref{eqn: mini-batch-sgd} and noticing $y_i = f(x_i;\theta^*)$, we obtain the linearized SGD: 
\vspace*{-.5em}
\begin{align}\label{eqn: linear-sgd}
\delta_{t+1} = \delta_t - \eta \nabla f(x_i;\theta^*) \nabla f(x_i;\theta^*)^T \delta_t,
\end{align}
where the high-order term is neglected.
This linearized SGD characterizes how $\delta_t$ evolves when $\theta_t$ is close to $\theta^*$.

\begin{definition}[Linear stability]\label{def: linear-stability}
Let $(\delta_t)_{t\in\NN}$ be the solution of the linearized SGD \eqref{eqn: linear-sgd}. 
A global minimum $\theta^*$ is said to be linearly stable if  $\|\EE[\delta_t\delta_t^T]\|_F\leq \|\EE[\delta_0\delta_0^T]\|_F$ for any $t\in \NN$ and initial distribution over $\delta_0$.
\end{definition}

The linear stability defined above measures the instability by using the second-order moment of deviations.
If $\theta^*$ is not linearly stable, it is unlikely that $\theta_t$ converges to $\theta^*$.
The following provides a necessary condition of linear stability, whose proof can be found in Appendix \ref{sec: proof-stability-1}.
\begin{proposition}\label{pro: stability-1}
If a global minimum $\theta^*$ is linearly stable, then 
$
    \tr(G(\theta^*))\leq 2/\eta.
$
\end{proposition}
% The following lemmas shows that the upper bound $2/\eta$ is optimal since it cannot be improved without imposing stronger assumption on the $\nabla f(x_i;\theta^*)$. 
% \begin{lemma}
% Let $g_i=\nabla f(x_i;\theta^*)$. Assume $\langle g_i,g_j\rangle=\delta_{i,j}$ and $\delta_0\in \mathrm{Span}\{g_1,\dots,g_n\}$. Then, if $\tr(G(\theta^*))>2/\eta$, we have $\|\delta_0\|=()$
% \end{lemma}
 This proposition implies that SGD tends to select minima, where the sharpness--as measured by the trace of Hessian--is bounded by $2/\eta$, independently of the model size and sample size.  This size-independence means that stability imposes an effective sharpness control no matter how over-paramterized the model is and this in turn 
 will yield an effective control on the model capacity as demonstrated in our subsequent generalization analysis. In contrast, for GD, the linear stability imposes a much weaker control: Only the largest eigenvalue of Hessian is bounded by $2/\eta$. 

\paragraph*{Comparison with existing works.} \citet{defossez2015averaged} derived the same upper bound of learning rate for  least square problems and studied its impact on the convergence of SGD. In contrast, our focus is understanding the  implication for regularization. Moreover, it should be stressed that our stability condition is derived by examing the linearized SGD but relevant for nonlinear SGD. \citet{ma2021linear} also derived an upper bound of $\tr(G(\theta))$ by examing the linear stability but their  bound  grows explicitly  with the model size. Specifically, \citet[Theorem 2]{ma2021linear} gives the bound $\tr(G)\leq 2p/\eta$, where $p$ is the number of parameters.

\subsection{Loss stability}
\label{sec: loss-stability}
In this section, we revise the {\em loss stability} defined in \citet{wu2022alignment}, which is applicable to a general SGD:
\begin{equation}
\theta_{t+1} = \theta_t - \eta (\nabla \erisk(\theta_t)+\xi_t),
\end{equation}
where $\xi_t$ denotes a general gradient noise that satisfies 
\begin{align}\label{eqn: noise-assumption}
\EE[\xi_t]=0,\quad \Sigma(\theta_t):=\EE[\xi_t\xi_t^T]=2\erisk(\theta_t)S(\theta_t).
\end{align}
Here $S(\theta)$ represents the loss-scaled noise covariance matrix.
This assumption of gradient noise implies that the noise magnitude is proportional to the loss value, which is naturally satisfied by the mini-batch SGD \eqref{eqn: mini-batch-sgd} as pointed out in \citet{mori2021logarithmic,wu2022alignment,wojtowytsch2021stochastic,feng2021inverse,liu2021noise}.

\begin{lemma}[One-step update]\label{lemma: loss-update}
Suppose $\erisk\in C^3(\RR^p)$. We have $\EE[\erisk(\theta_{t+1})]\geq \eta^2 \tr[H(\theta_t)S(\theta_t)] \erisk(\theta_t) + O(\eta^3)$
\end{lemma}
\vspace*{-1em}
\begin{proof}
By definition, we have 
\begin{align*}
\vspace*{-.2em}
\hspace*{-.7em}\erisk(\theta_{t+1}) &= \erisk(\theta_t - \eta \nabla \erisk(\theta_t) - \eta \xi_t)\\
\hspace*{-.7em}&= \erisk(\theta_t - \eta \nabla \erisk(\theta_t)) + \<\nabla \erisk(\theta_t - \eta \nabla \erisk(\theta_t)), - \eta \xi_t\> \\ 
\hspace*{-.5em}&\qquad + \frac{\eta^2}{2} \xi_t^T H(\theta_t - \eta \nabla \erisk(\theta_t))\xi_t + O(\eta^3).
\vspace*{-.5em}
\end{align*}
Taking  expectation {\em w.r.t.} $\xi_t$ and using $H(\theta_t-\eta \nabla \erisk(\theta_t)) = H(\theta_t) + O(\eta)$ and $\erisk(\theta_t - \eta \nabla \erisk(\theta_t))\geq 0$ gives 
\begin{align}\label{eqn: 1}
\notag  \EE[\erisk(\theta_{t+1})] &\geq  \frac{\eta^2}{2} \tr[H(\theta_t)\Sigma(\theta_t)]+ O(\eta^3)\\
  &=\eta^2 \erisk(\theta_t)\tr[H(\theta_t)S(\theta_t)] + O(\eta^3),
\end{align}
where the last step follows from \eqref{eqn: noise-assumption}.
\end{proof}

This lemma implies that $\tr[H(\theta_t)S(\theta_t)]$ determines the local stability if ignoring the higher-order term. Specifically, if the loss $\erisk(\theta_t)$ is sufficiently small such that $H(\theta_t)\approx G(\theta_t)$, then for $\EE[\erisk(\theta_{t+1})]\leq \erisk(\theta_t)$ to hold, a necessary condition is 
$
  \tr[G(\theta_t)S(\theta_t)]\leq 1/\eta^2.
$
This condition can be converted to a sharpness control by assuming
\vspace*{-.8em}
\begin{equation}\label{def: alignment}
\mu(\theta) := \frac{\tr(G(\theta)S(\theta))}{\|G(\theta)\|_F^2}\geq \mu_0.
\vspace*{-.5em}
\end{equation}
By treating $G(\theta)\approx H(\theta)$, $\mu(\theta)$ can be interpreted as a factor that quantifies the (loss-scaled) strength of alignment between the noise covariance and local Hessian. For mini-batch SGD \eqref{eqn: mini-batch-sgd}, \citet{wu2022alignment} has shown that there exist a size-independent constant $\mu_0>0$  such that $\mu(\theta)\geq \mu_0$ for neural networks. We refer to \citet{wu2022alignment} for more discussions on this alignment factor. 

% This following proposition formalizes this procedure, whose

\begin{proposition}\label{pro: stability-2}
Assume $\erisk\in C^3(\RR^p),\mu(\theta)\geq \mu_0$. Let $\cQ_{\varepsilon,\eta} = \{\theta: \erisk(\theta)\leq \varepsilon, \|G(\theta)\|_F>\sqrt{1/\mu_0}/\eta\}$. If $\{\theta_\tau\}_{\tau=0}^t\in \cQ_{\varepsilon,\eta}$, then $\EE[\erisk(\theta_{t+1})]\geq \gamma^t \erisk(\theta_0)+\frac{\gamma^t-1}{\gamma-1}O(\eta^3+\eta^2\varepsilon^{3/2})$ with $\gamma>1$.
\end{proposition}
The proof can be found in Appendix \ref{sec: proof-stability-2}.
This proposition shows that  SGD will escape from a low-loss region {\em exponentially fast} (measured by the loss value) if the the landscape there is too sharp in terms of the Frobenius norm of Hessian. Specifically,  SGD can only stay/travel in the region where
$
\|G(\theta)\|_F\leq \sqrt{1/\mu_0}/\eta.
$

The above analysis extends \citet{wu2022alignment} in two aspects. First, our analysis does not need $\theta_t$  to be close to a global minimum $\theta^*$, implying that the loss stability is relevant even if SGD has not converge. This is consistent with the numerical experiments in \citet{wu2022alignment}, which shows that the upper bound of Frobenius norm of Hessian matrix holds for the entire training process.
Second, our analysis is applicable to general SGD where the gradient noise  does not necessarily come from the mini-match sampling. For instance, one can consider the Langevin dynamics with $S(\theta)=G(\theta)$, which has been tested in \citet{zhu2019anisotropic} to have similar generalization properties as mini-batch SGD.  In contrast, \citet{wu2022alignment}   only considered the mini-batch SGD and required $\theta^*$ is close to $\theta_t$ in the sense that $\|\theta_t-\theta^*\|=o(1)$.
% We remark that different from the linear sta
% It should be stressed that Proposition \ref{pro: stability-2} holds for general SGD. One particular example is taking 
% % , and 

\begin{remark}
Note that different from the linear stability (Definition \ref{def: linear-stability}), the loss stability measures the stability by using the changes of loss.  In \citet{wu2022alignment}, this stability is referred to as ``linear stability''. However, based on our preceding explanations, we propose the term ``loss stability'' as a more fitting term. Furthermore, we will specifically refer to  Definition \eqref{def: linear-stability} as ``linear stability'', acknowledging that it exclusively holds for the linearized SGD.
\end{remark}

\subsection{The comparison between two types of stability}
\label{sec: stability-remark}

The linear stability  is defined by examining the linearized SGD \eqref{eqn: linear-sgd}, which is validated only if $\theta_t$ is sufficiently close to a global minimum $\theta^*$. 
In contrast, the  loss stability measures the stability by inspecting if the loss grows exponentially, which 
 is applicable even if $\theta_t$ does not converge. 
 % Specifically, our loss stability analysis in Section \ref{sec: loss-stability}  works as long as $\eta$ is relatively small. 
 In terms of sharpness control, linear stability and loss stability impose size-independent controls on $\tr(G(\theta))$ and $\|G(\theta)\|_F$, respectively. The former is stronger since $\|G(\theta)\|_F\leq \tr(G(\theta))$.
To summarize, linear stability yields a stronger sharpness control but requires the dynamics to be sufficiently close to a global minimum;  loss stability is generally relevant but imposes a weaker  sharpness control. 
When the loss-stability is satisfied, SGD may travel in a low-loss region without convergence and the convergence to a global minimum requires the stronger condition of linear stability to be satisfied. 
% In addition, to obtain a closed-form stability condition using the Fisher matrix, we only need $\erisk(\theta_t)$ to be relatively small, for which $\theta_t$ may be away from global minima. 

Then a natural question is: Which type of stability characterizes the actual dynamical behavior of SGD better? 
The answer will depends on the problem and training stages. 
\begin{itemize}
\item In training practical models, large-LR SGD often takes many iterations to stay in a low-loss region without reaching a global minimum. During these stages, the loss keeps nearly unchanged; and thus, the condition of loss stability must be met but the condition of linear stability is not necessarily to be satisfied. Indeed, the empirical studies by
\citet{wu2022alignment} has demonstrated the relevance of loss stability in this situation.

\item In this paper, we focus on detailed analysis of simple models: two-layer ReLU networks and diagonal linear networks, for which we empirically find that linear stability is more relevant. Specifically, the upper bound $2/\eta$ is  close to the actual trace of Hessian. While the condition of loss stability is also satisfied, the resulting bound is much looser. 
 These observations are not unexpected  as for these simple models, large-LR SGD  always converges to zero loss stably, implying the condition of linear stability must be satisfied. 
\end{itemize}

\section{Two-layer ReLU networks}
\label{sec: relu-net}

We first consider the two-layer ReLU network: $
f(x;\theta) = \sum_{j=1}^m a_j \sigma(w_j^Tx),
$
where $a_j\in\RR,w_j\in\RR^d$,  $m$ denotes the network width, and $\sigma(t)=\max(t,0)$.
In this section, we assume the input distribution to be $\rho=\unif(\sqrt{d}\SS^{d-1})$.

% the specific assumption of $\rho$ can be potentially replaced by any sub-Gaussian distribution with positive density.

Define the weighted $\ell_2$ norm $\|\theta\|_{2,q}:=\sum_{j}(\|w_j\|^2+qa_j^2)$, where $q>0$ is the weight factor.
The following theorem shows that all three sharpness are equivalent to  $\ell_{2,q}$ norms of parameters and the only difference is the weight factor. The proof  can be found in Appendix \ref{sec: proof-2lnn-sharpness}

\begin{theorem}\label{thm: sharpness-relu-landscape}
 For any $\delta\in (0,1)$, let $N(d,\delta)=\inf\{n\in\NN : d\log(n/\delta)/n\leq 1\}$.
 \begin{itemize}
 \item If $n\gtrsim N(d,\delta)$, then \wp  $1-\delta$  we have 
\[
     \|G(\theta)\|_F\sim \|\theta\|_{2,\sqrt{d}},\,\, \tr(G(\theta))\sim \|\theta\|_{2,d}.
\]
\item  If $n\gtrsim dN(d,\delta)$, then \wp  $1-\delta$,   $\|G(\theta)\|_2\sim \|\theta\|_{2,1}$.
  \end{itemize} 
\end{theorem}

\begin{remark}
Sharpness is a  data-dependent quantity since it measures the local curvature of \emph{empirical landscape}. In contrast, a weighted $\ell_2$ norm of parameters is data independent. The equivalence shown in Theorem \ref{thm: sharpness-relu-landscape} is possible because we assume $\rho$  to be isotropic. A  question of more interest would be to exploit the effect of data dependence by making an anisotropic assumption on $\rho$, which we leave to the future work.
\end{remark}

For ReLU networks, it is well-known that the generalization gap can be controlled by the path norm \citep{neyshabur2015norm,weinan2021barron}
$
\|\theta\|_{\cP}:=\sum_j |a_j|\|w_j\|.
$
By the AM-GM inequality, we have 
\begin{align}\label{eqn: 321}
\notag \|\theta\|_{2,q}&=\sum_j (\|w_j\|^2 + q a_j^2) \\ 
&\geq 2\sqrt{q}\sum_j |a_j|\|w_j\| =2\sqrt{q}\|\theta\|_{\cP}. 
\end{align}

\vspace*{-1.2em}
This implies that  weight $\ell_2$ norms can bound the generalization gap although it is not rescaling invariant.

% Therefore, Theorem \ref{thm: sharpness-relu-landscape} implies that all the three sharpnesses can effectively control the path norm, thereby the generalization gap. Thus in this case, for both SGD and GD, the linearly stable minima provably generalize. 

% \begin{theorem}\label{thm: a-posteriori-bound}
% Suppose $\sup_{x\in \cX}|f^*(x)|\leq 1$ and $\gamma\geq 1$. If $\htheta$ is a global minimum satisfying $\|\htheta\|_{\cP}\leq \gamma$, then 
% $
%   \risk(\htheta)\leq \frac{C_{n,d,\delta}\gamma^2}{n}
% $
% with $C_{n,d,\delta}=\log^3(n)d+\log(1/\delta)$.
% \end{theorem}
% \begin{proof}
% Let $\cF_{\gamma}=\{f(\cdot;\theta):\|\theta\|_{\cP}\leq \gamma\}$. Then, it is easy to show that $\rad(\cF_\gamma)\lesssim \sqrt{d}\gamma/\sqrt{n}$. By the boundness assumption, the loss function $\phi(t)=t^2/2$ satisfies $|\phi''(t)|\lesssim 1$ and $|\phi(t)|\lesssim \gamma^2$. Then, by Theorem \ref{thm: fast-rates-smoothloss}, we have 
% $
%   \risk(\htheta)\lesssim \frac{\log^3(n)d+ \log(1/\delta)}{n}\gamma^2.
% $
% \end{proof}

\begin{theorem}\label{thm: 2lnn-generalization-bound}
For SGD and GD with the same LR $\eta$, denote by $\htheta_{\text{sgd}}$ and $\htheta_{\text{gd}}$ the linearly stable minimum of SGD and GD, respectively.
Suppose $\sup_{x\in \cX}|f^*(x)|\leq 1$.  For any $\delta\in (0,1)$, if $n\gtrsim d N(d,\delta)$, then the following holds \wp at least $1-\delta$
\[
  \risk(\htheta_{\mathrm{sgd}})\lesssim \frac{B}{\eta^2 n},\qquad \risk(\htheta_{\mathrm{gd}})\lesssim \frac{B_{}d}{\eta^2 n},
\]
where $B=\log^3n + \log(1/\delta)$.
\vspace*{-1em}
\end{theorem}

\paragraph*{Proof idea.}
The complete proof can be found in Appendix \ref{sec: proof-2lnn-gengap}. Here we provide a sketch of proof idea. For $\htheta_{\mathrm{sgd}}$, the generalization gap can be informally bounded as follows
\begin{align}
 \notag   \text{gen-gap}(\htheta_{\mathrm{sgd}}) &\stackrel{(a)}{\lesssim} \frac{d\|\htheta_{\mathrm{sgd}}\|^2_{\cP}}{n}\stackrel{(b)}{\lesssim} \frac{\|\htheta_{\mathrm{sgd}}\|^2_{2,d}}{n}\\ 
    &\stackrel{(c)}{\lesssim} \frac{\tr^2(G(\htheta_{\mathrm{sgd}}))}{n}\stackrel{(d)}{\leq} \frac{1/\eta^2}{n},
\end{align}
where $(a)$ follows from the path norm-based generalization bound (Proposition \ref{pro: 2lnn-sharpness-genbound}); $(b)$ follows from \eqref{eqn: 321}; $(c)$ follows from  Theorem \ref{thm: sharpness-relu-landscape}; $(d)$ is due to the stability condition (Proposition \ref{pro: stability-1}). 

This theorem shows  that  stable minima provably generalize, no matter how over-parameterized the model is. This suggests that the stability-induced regularization is strong enough to eliminate the potential overfitting caused by over-parameterization. In addition, with the same LR, stable minima of SGD  generalize better than that of GD. However, it is more fair to compare SGD with LR $\eta$ and GD with LR $\sqrt{d}\eta$. Thus, we will use this LR choice in our experimental analysis for a fair comparison between SGD and GD.

% One may argue that GD with the larger learning rate $d\eta$ will have the same bound of the path norm and the generalization gap. However, in practice, SGD with learning rate $\eta$  converges but GD with  learning rate $ d \eta$ may not converge, in particular when $d\gg 1$.
\vspace*{-.5em}
\paragraph*{Comparison with existing works.}
\citet{mulayoff2021implicit} conducted a similar analysis for two-layer ReLU networks, which, however, is limited to GD and the univariate case.  Another closely related work is  \citet{ma2021linear}, which established a generalization bound of linearly stable minima of SGD for ReLU networks but the bound  suffers from the curse of dimensionality. One of the reasons is that the upper bound of the trace of Hessian derived in \citet{ma2021linear} depends on the model size explicitly.  
In contrast,  our generalization bounds are effective in high dimensions and hold for both SGD and GD.

\vspace*{-.2em}
\subsection{Numerical validations}
\label{sec: 2lnn-experiment}
\vspace*{-.3em}

Consider $f^*(x)=\sum_{i=1}^{k}\sigma(v_i^Tx)$ with $v_i\stackrel{iid}{\sim} \unif(\SS^{d-1})$ and  $f(x;\theta)=\sum_{j=1}^m a_j\sigma(w_j^Tx)$. We set $k=10, d=100,m=100$, and the sample size $n=300$. With this choice, the total number of parameters is $p=(d+1)m=10100$ and thus, we are examining a highly over-parameterized case where $p\gg n$. We consider a large initialization: $a_j\sim\cN(0,1)$ and $w_j\sim \cN(0, I_d/\sqrt{d})$, with which the path norm at initialization: $\|\theta\|_{\cP}\sim m$, growing linearly with the network width. This large initialization excludes the small initialization effect and one must rely on the stability-induced regularization to select minima with small path norms. In addition, gradient clipping will be applied to stabilize the training if SGD/GD blows up initially.

\paragraph*{The effect of gradient clipping.}
Figure \ref{fig: relu-net-sgd-process} shows the dynamical process of SGD with gradient clipping, where $\eta=1/\sqrt{d}$ and the clipping threshold  $\delta=1$. One can see that the gradient clipping is automatically switched off since around $4000$ iterations. After that, SGD can stably converge to a global minimum without clipping operations. This implies that around the convergent minimum, linear stability should be satisfied and consequently, it is not surprising to observe that $\tr(G(\theta_t))\leq 2/\eta$ when $\theta_t$ nearly converge. 
Another interesting observation is that during the whole training process, $\tr(G(\theta_t))$ keeps decreasing, which in turn causes the continued decreasing of path norm. This phenomenon cannot be explained by the stability condition and one should delve into the dynamical process of SGD. We  refer to \citet{blanc2020implicit} for a potential explanation.

\begin{figure}[!h]
\centering
\includegraphics[width=0.165\textwidth]{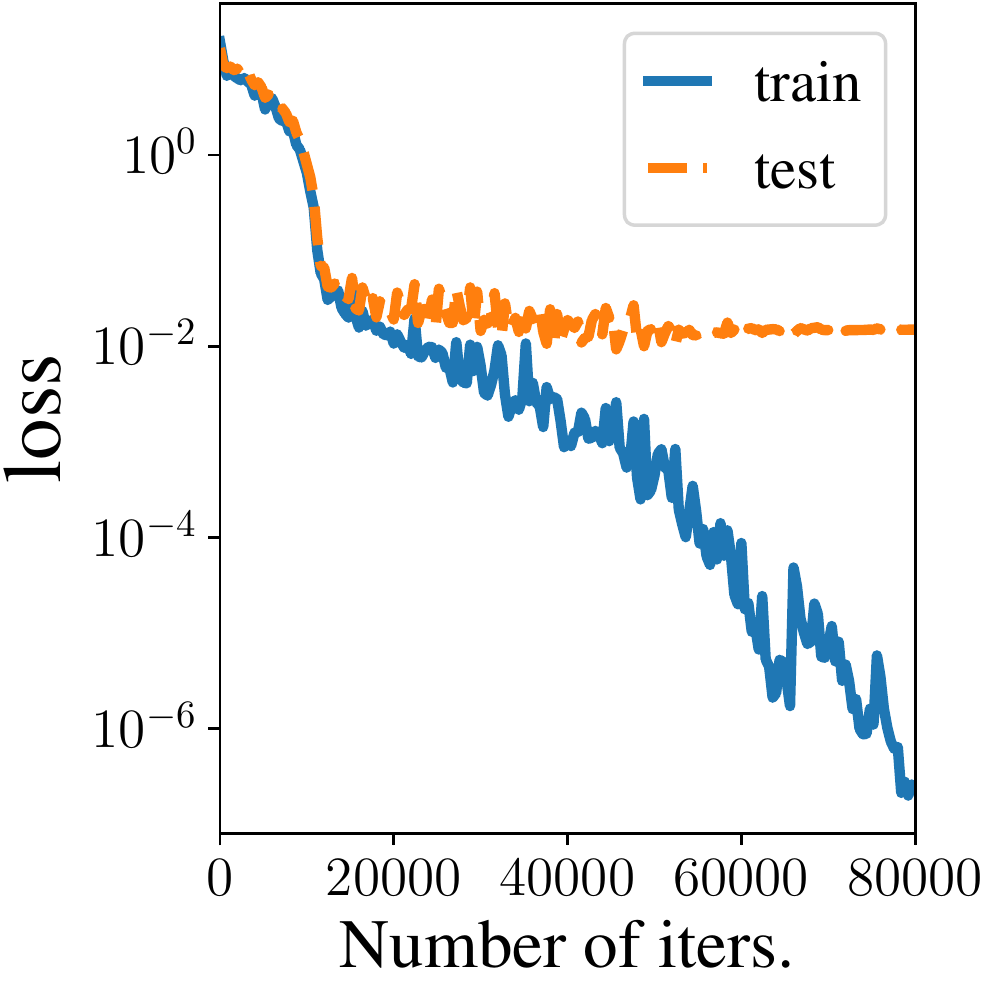}
\includegraphics[width=0.17\textwidth]{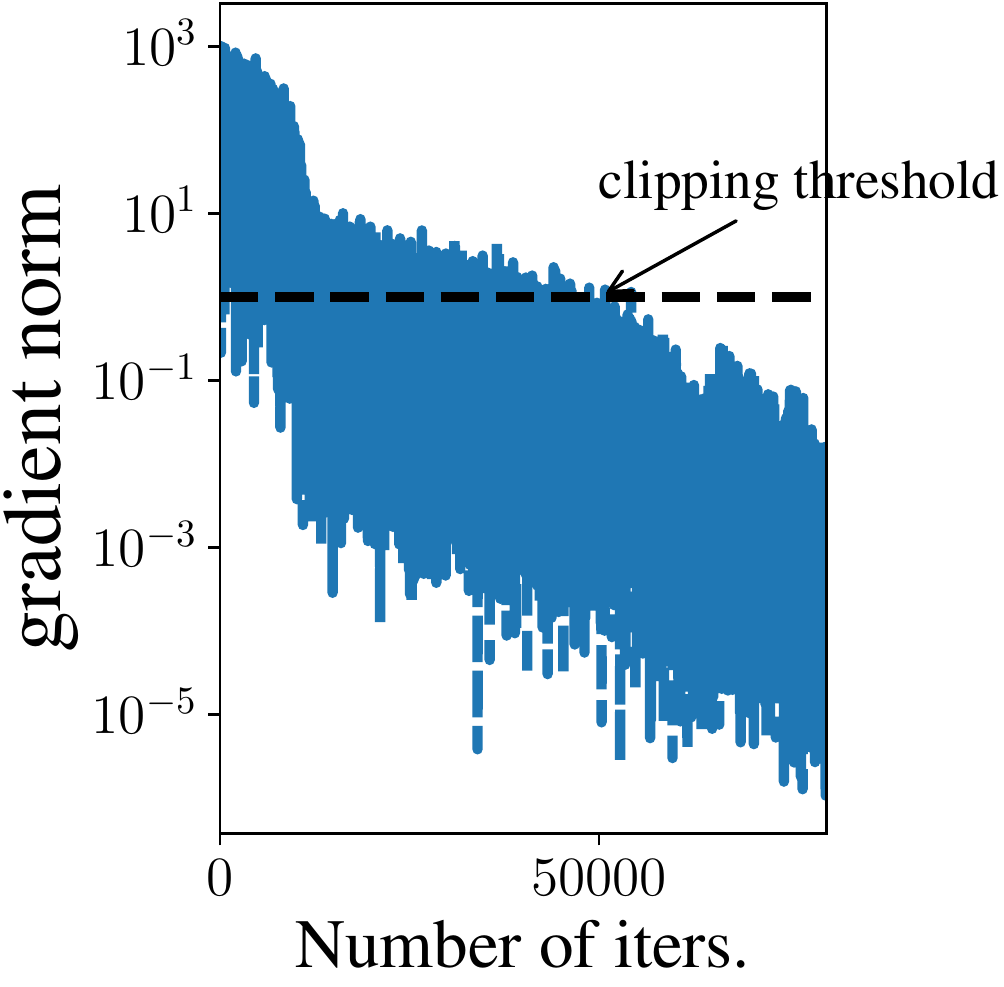}
\includegraphics[width=0.125\textwidth]{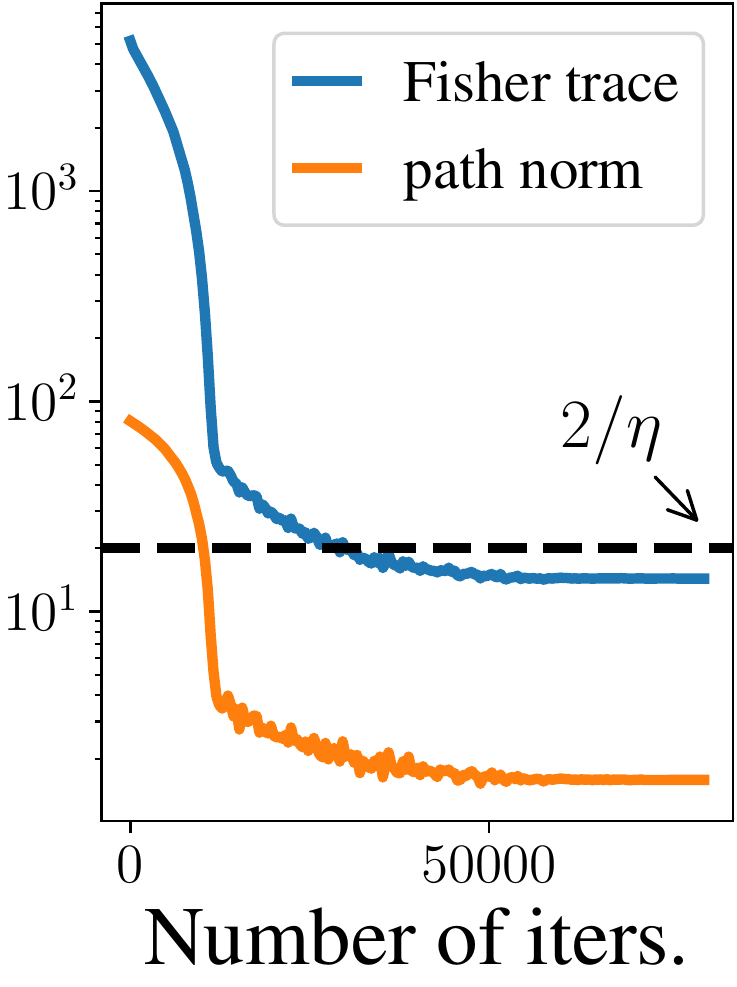}
\vspace*{-.4em}
\caption{\small The training process of SGD with gradient clipping. The gradient clipping is automatically switched off in the late phase of training. The trace of Fisher matrix keeps decreasing  until it becomes lower than $2/\eta$ and meanwhile, the path norm also keeps decreasing, which is consistent with Theorem \ref{thm: sharpness-relu-landscape}.}
\label{fig: relu-net-sgd-process}
\end{figure}

\begin{figure}[!h]
\subfigcapskip = -.4em
\centering
\subfigure[]{\label{fig: relunet-1}
\includegraphics[width=0.185\textwidth]{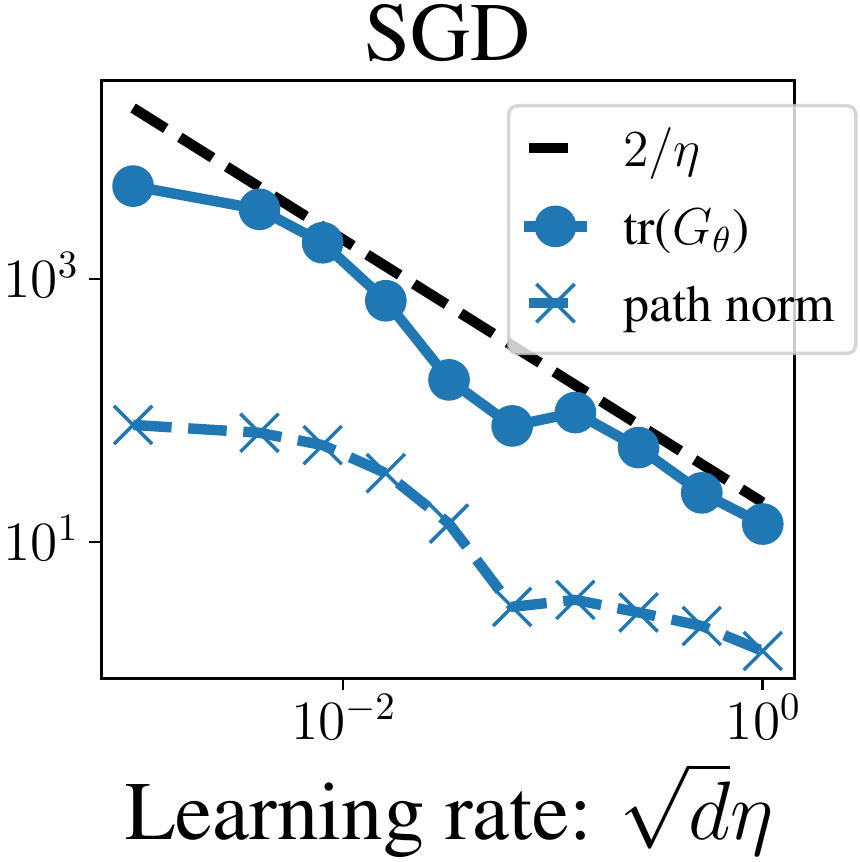}
\hspace*{1em}
\includegraphics[width=0.18\textwidth]{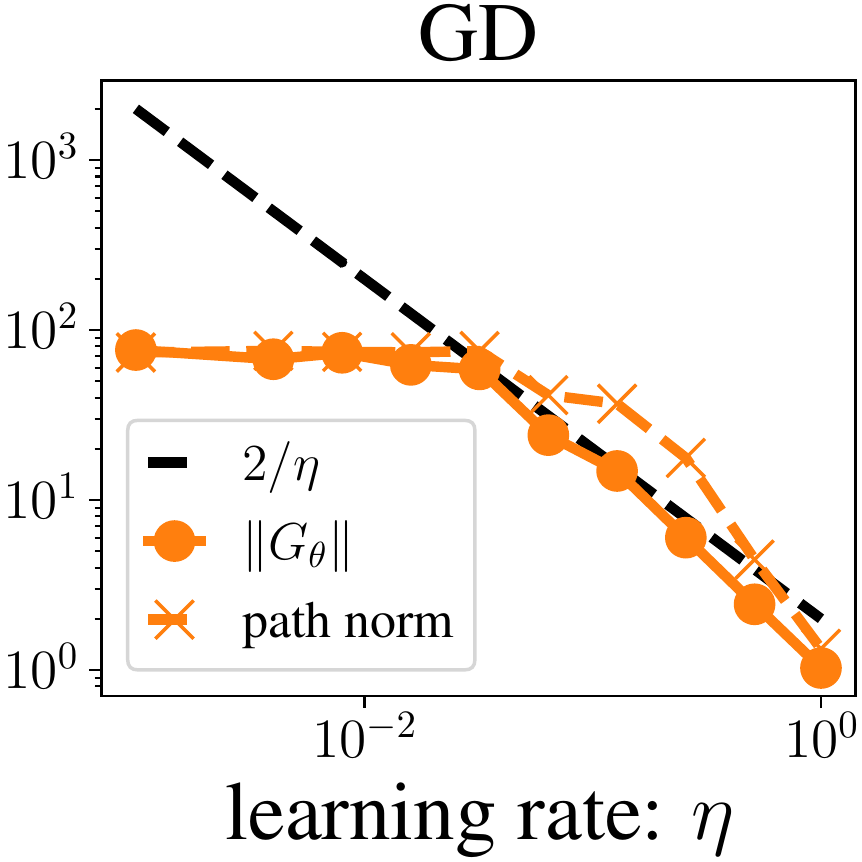}
}
% \vspace*{-1em}
\subfigure[]{\label{fig: relunet-2}
\includegraphics[width=0.205\textwidth]{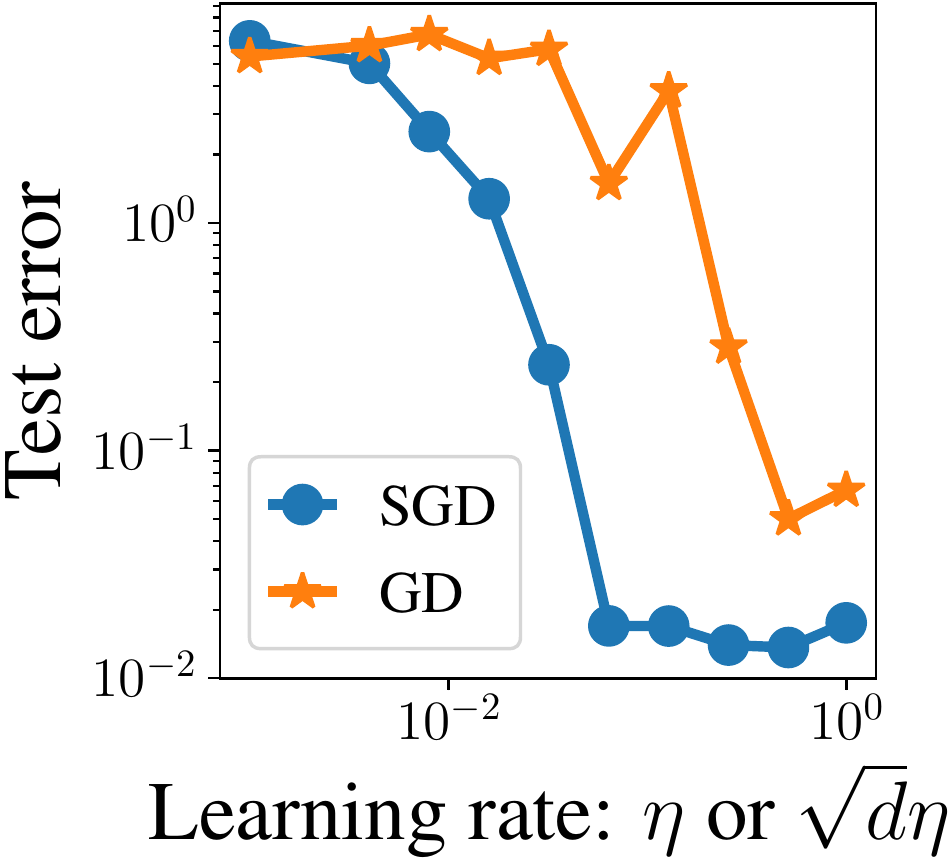}
\hspace*{.5em}
\includegraphics[width=0.2\textwidth]{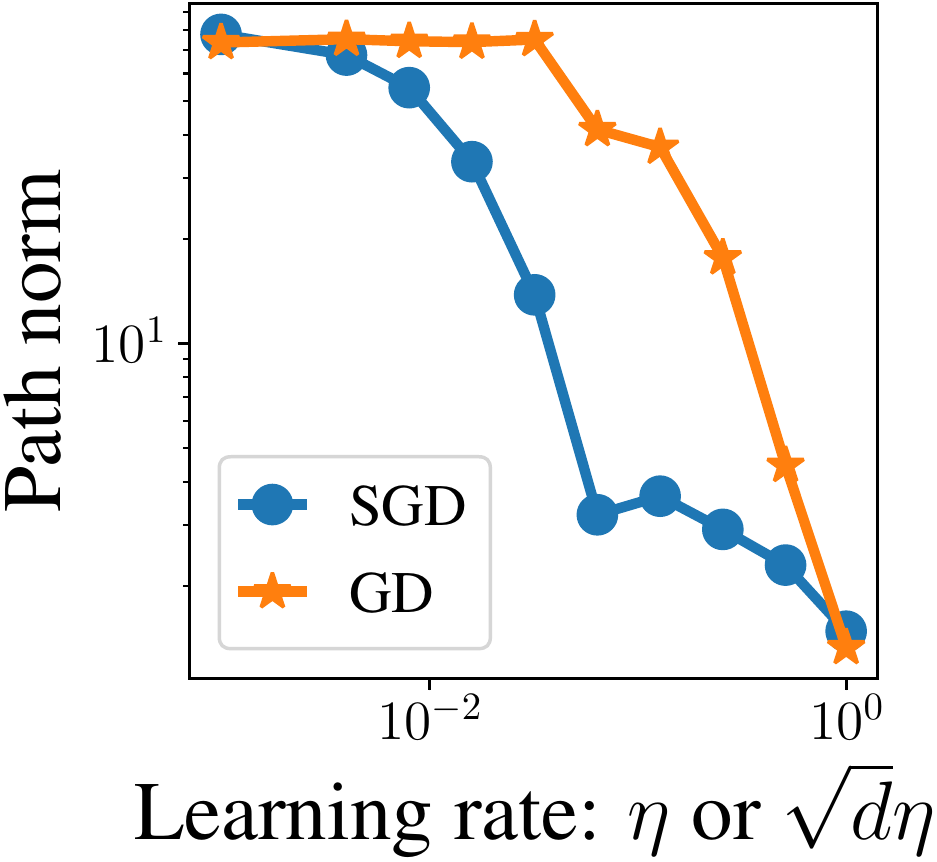}
}
\vspace*{-1em} 
\caption{\small (a) The sharpness and path norm vs. LR. (b) The  test performance vs. LR.
For a fair comparison, we compare SGD with LR  $\eta/\sqrt{d}$ and GD with LR $\eta$.
}
\label{fig: relu-net-sgd-vs-gd}
\vspace*{-1em}
\end{figure}

\paragraph*{The sharpness.}
Figure \ref{fig: relunet-1} shows how the sharpness and path norm of minima selected by SGD and GD changes with the LR. For SGD, $\tr(G(\theta))$ keeps decreasing but close to the upper bound $2/\eta$. Consequently, the path norm also keeps decreasing. This is consistent with the predictions of linear stability analysis and Theorem \ref{thm: sharpness-relu-landscape}. 
For GD, $\|G(\theta)\|_2$ keeps close to the upper bound $2/\eta$ when the LR is sufficiently large. When the LR is small, the actual sharpness is away from the upper bound. These observations suggest that the impact of stability-induced  regularization is particularly significant in the large LR regime.

\vspace*{-.8em}
\paragraph*{The test performance.}
Figure \ref{fig: relunet-2} shows that the test performance is continually improved for both SGD and GD as increasing the LR, which again confirms the prediction of Theorem \ref{thm: 2lnn-generalization-bound}.  One observation of more interest is that 
Figure \ref{fig: relunet-2} shows that with the fair choice of LR, SGD still generalizes better than GD. This beyond what Theorem \ref{thm: 2lnn-generalization-bound} can explain since the generalization bounds are the same for SGD and GD in such a case. In addition, when the LR is overly large, SGD still generalizes better although its path norm becomes larger than that of GD.  
A potential explanation is that the noise  drives SGD towards better minima  with certain mechanism beyond dynamical stability. We leave this to future work.

% \paragraph*{Balancedness.} We use $\gamma(\theta) = \sum_{j}a_j^2/\sum_{j}\|w_j\|^2$ to measure the balanceness of inner- and outer-layer weights. For full-batch GD, the dynamical stability and Theorem \ref{thm: sharpness-relu-landscape} suggests that at the stable minima of GD,  $\|G(\theta)\|\sim (\sum_j a_j^2 +\|w_j\|^2)$ is controlled. Therefore, it is expected that GD minima are balanced in the sense that $\gamma(\htheta_{\mathrm{gd}})\approx 1$. On the other hand, the linear stability of SGD implies that at the stable minima, $\tr(G(\theta))\sim \sum_{j}(d a_j^2 + \|w_j\|^2)$ is small. This implies that  SGD should favor minima where $\gamma(\htheta_{\mathrm{sgd}})\approx 1/d$.
% \vspace*{-.3em}
\section{Diagonal linear networks}
% \vspace*{-.2em}
\label{sec: diagonal nets}

\vspace*{-.2em}
Consider the two-layer diagonal linear network  
\begin{equation}
\vspace*{-.5em}
f(x;\theta)=\langle a\odot b,x\rangle,
\vspace*{-.1em}
\end{equation}
where $a,b\in\RR^d$, $\theta=(a,b)\in\RR^{2d}$, and $\odot$ denotes the element-wise multiplication. Despite its simplicity, this model has been widely used in theoretical analysis to demonstrate particular properties of SGD in training neural networks \cite{woodworth2020kernel,gissin2019implicit,pesme2021implicit,nacson2022implicit}. Note that this model  can only represent linear predictors and we will use $\beta=a\odot b$  to denote the effective coefficients. In this section, we make the following assumption on $\rho$. 
\begin{assumption}\label{assumption: input-1}
Let $X\sim \rho$. Assume  $\EE[XX^T]=I_d$ and $X$ is sub-Gaussian, i.e., $\|u^TX\|_{\psi_2}\lesssim 1$ for any $u\in\SS^{d-1}$.
\end{assumption}
Here $\|\cdot\|_{\psi_2}$ denotes the sub-Gaussian norm (we refer to Appendix \ref{app: concentration} for details) and one typical example that satisfies the above assumption is $\cN(0,I_d)$ and $\mathrm{Unif}([-1,1]^d)$.

\begin{theorem}\label{thm: diagnet-sharpness}
Suppose Assumption \ref{assumption: input-1} holds. Let $\alpha = a\odot a + b\odot b$. Let  $\delta\in (0,1)$  be the failure probability.
\begin{itemize}
\item If $r_n  = \sqrt{(d+\log(1/\delta))/n}\leq 1$, then \wp  $1-\delta$ that 
\[
(1-r_n) \|\alpha\|_\infty \leq \|G(\theta)\|_2 \leq (1+r_n) \|\alpha\|_\infty.
\]
\item If $\varepsilon_n = \sqrt{\log(d/\delta)/n}\leq 1$. Then, \wp  $1-\delta$ that 
\begin{align*}
\hspace*{-1.5em} (1-\varepsilon_n)\|\alpha\|_2&\leq \|G(\theta)\|_F \leq \varepsilon_n \|\alpha\|_1+(1+2\varepsilon_n)\|\alpha\|_2\\
\hspace*{-1.5em}(1-\varepsilon_n)\|\alpha\|_1 &\leq \tr(G(\theta)) \leq (1+\varepsilon_n)\|\alpha\|_1.
\end{align*} 
\end{itemize}
\end{theorem}
This theorem establishes the equivalence between the sharpness and parameter norms, whose proof is deferred to  Appendix \ref{sec: proof-diagnet-sharpness}.  It is worth noting that the cases of Frobenius norm and trace hold in the highly over-parameterized regime: $n \sim \log(d/\delta)$.

\begin{figure*}[!ht]
\centering
\subfigcapskip = -.8em
\subfigure[]{\label{fig: diag-net-a}
\includegraphics[width=0.25\textwidth]{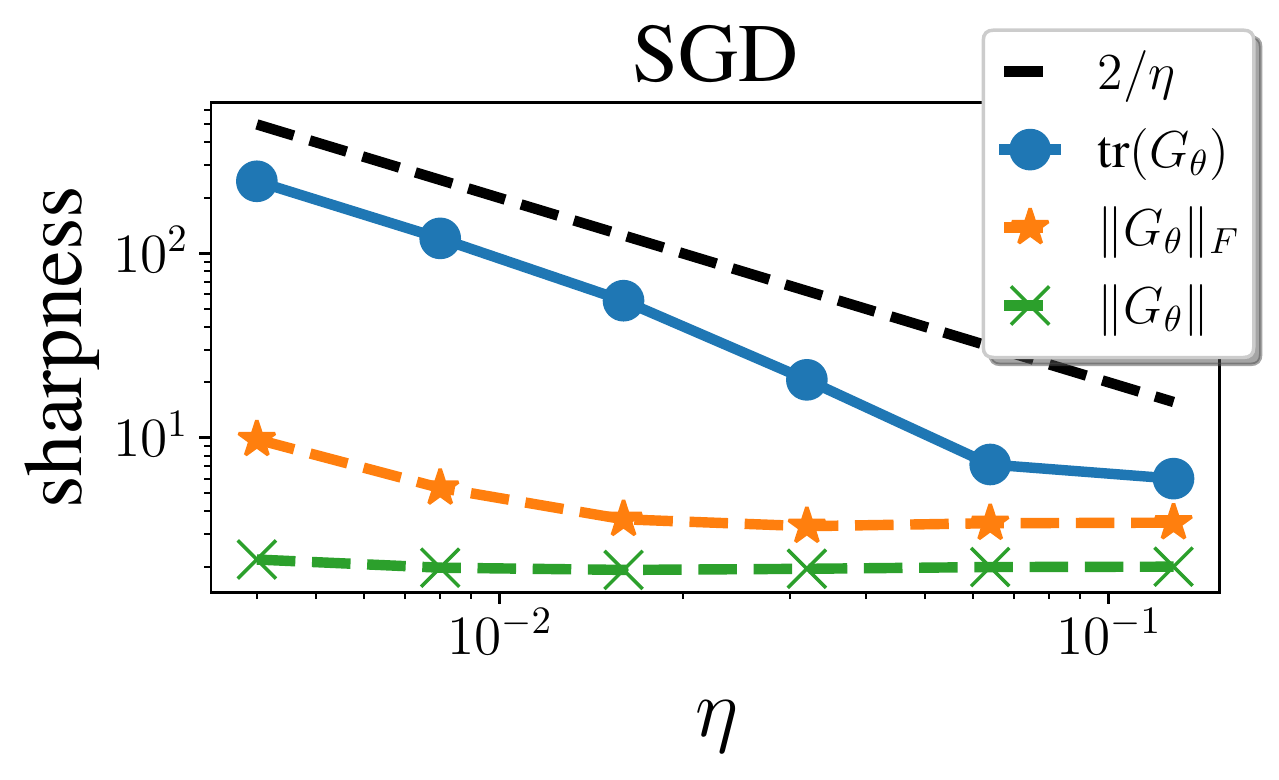}
\hspace*{-.5em}
\includegraphics[width=0.25\textwidth]{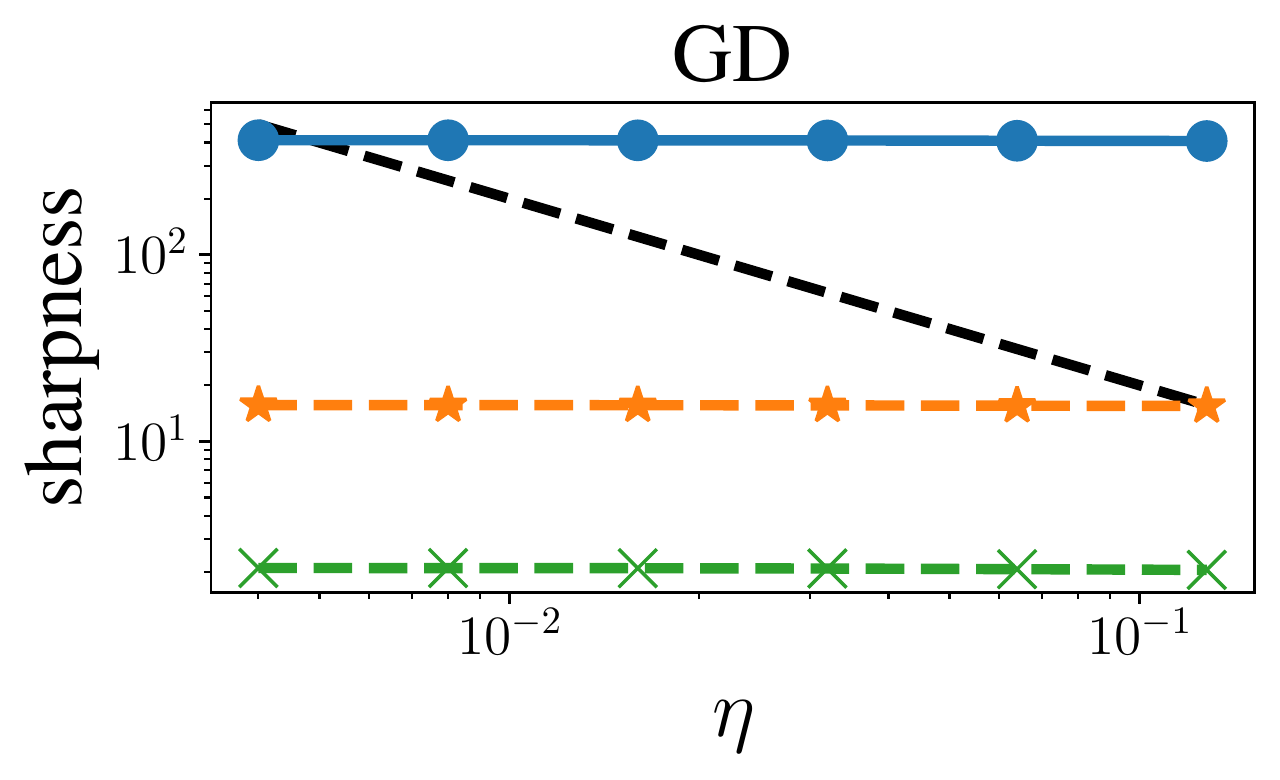}
}
\hspace*{-1em}
\subfigure[]{\label{fig: diag-net-b}
\includegraphics[width=0.235\textwidth]{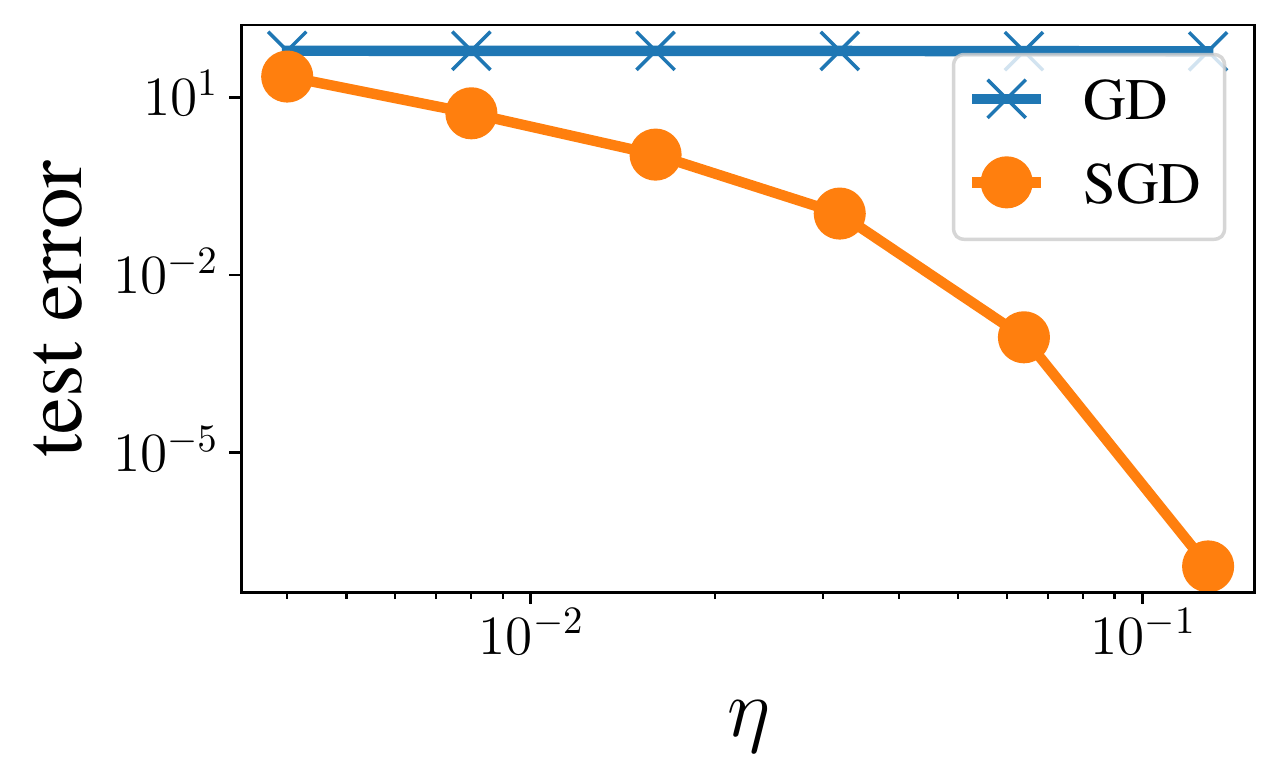}
}
\subfigure[]{\label{fig: diag-net-c}
    \includegraphics[width=0.235\textwidth]{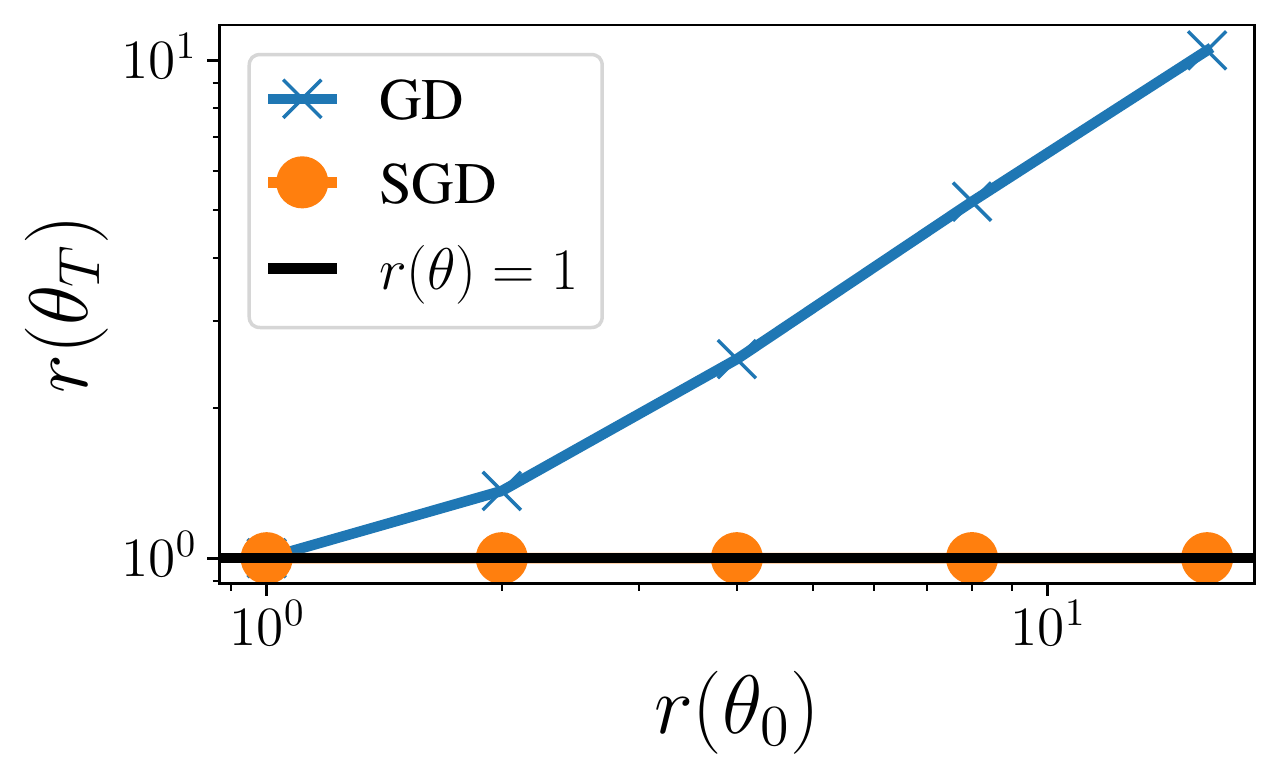}
}
\vspace*{-.8em}
\caption{\small (a) How the sharpness of minima found by SGD and GD changes with the learning rate. We see that for SGD, the upper bound $2/\eta$ provides a quite sharp  estimate of the actual trace of Fisher matrix up to a multiplicative constant.  In contrast, for GD, the sharpness barely changes as increasing the learning rate. (b) The comparison of test performance between SGD and GD for varying learning rates. (c) Demonstrate the balancing effect of SGD, where the unbalancedness is measured by $r(\theta)=0.5\|\alpha\|_2/\|\beta\|_1$. The horizontal and vertical axises correspond to the unbalance at initialization and convergence, respectively.
}
\label{fig: diag-net-1}
\vspace*{-1em}
\end{figure*}

Theorem \ref{thm: diagnet-sharpness} shows that with a high probability,  $\|G(\theta)\|_2$ is equivalent to $\max_{j}(a_j^2+b_j^2)$, which unfortunately cannot provide an  effective capacity control for  the linear predictor: $(a\odot b)^Tx$.  Furthermore, the stability of GD only imposes a size-independent control on $\|G(\theta)\|_2$. Thus, we can conclude that the stability-induced regularization of GD is not strong enough to help find generalizable minima. In contrast,  the stable minima of SGD provably generalize well, which is explained as follows. Noting
\begin{equation}\label{eqn: 111}
\begin{aligned}
\|\alpha\|_1&=\sum_j (a_j^2+b_j^2)\geq 2\sum_j |a_jb_j|=2\|\beta\|_1\\ 
\|\alpha\|_2^2&=\sum_{j} (a_j^2+b_j^2)^2\geq 4\sum_j (a_jb_j)^2 = 4\|\beta\|_2^2
\end{aligned}
\end{equation}
and applying Theorem \ref{thm: diagnet-sharpness}, we can conclude that the linear stability and loss stability can control the $\ell_1$ and $\ell_2$ norm of effective coefficients, respectively, which yield  effective capacity controls for the linear predictor.   Specifically, the following theorem formalizes this observation for the case of linear stability and the proof is deferred to Appendix \ref{sec: proof-diagnet-gen-gap}.

\begin{theorem}\label{thm: diagnet-gen-gap}
Suppose $\rho=\unif([-1,1]^d)$  and $f^*(x)=\beta_*^Tx$. Let $\htheta=(\ha,\hb)$ be a global minimum that is linearly stable for SGD \eqref{eqn: mini-batch-sgd} with LR $\eta$. Then,  for any $\delta\in (0,1)$, if $n\gtrsim \log(d/\delta)$, then \wp $1-\delta$ we have $\|\ha\odot \hb\|_{1}\lesssim 1/\eta$ and 
\[
  \risk(\htheta)\lesssim \frac{(1/\eta)^2\log^2(n)\log(d)}{n} + \frac{(\|\beta_*\|_1+\frac{1}{\eta})^2\log(1/\delta)}{n}.
\]
\end{theorem}
This theorem shows that SGD selects minima with the $\ell_1$ norm bounded by $1/\eta$. As long as the LR $\eta$ is sufficiently large and  $\|\beta_*\|_1=O(1)$, the minima found by SGD generalizes well.  \citet{woodworth2020kernel} showed that for this model, gradient flow converges to the minimum $\ell_1$ norm solutions when a (near-)zero initialization is used.   Nevertheless, we reveal that the linear stability of large-LR SGD has a similar effect, which is independent of the initialization scale.

\vspace*{-.4em}
\paragraph*{Comparison with \citet{nacson2022implicit}.} \citet{nacson2022implicit} obtained a similar result for  GD but it crucially relies on the non-centered data assumption: All coordinates of $\EE[X]$ are nonzero. In contrast, our analysis does not need this assumption and moreover, can explain why SGD generalizes better than GD. We also point out that under the non-centered data assumption, the stability-induced regularization might not be able to distinguish SGD and GD as both control the $\ell_1$ norm of effective coefficients.

% Note that over large learning rate can cause that there does not exist stable minima in the landscape and consequently, SGD with that learning rate is unable to converge. 

\vspace*{-.4em}
\paragraph*{The balancing effect.}
Another interesting consequence is that SGD tends to select balanced solutions where $a_j^2\approx b_j^2$ for any $j\in [n]$. This is because minimizing the trace and Frobenius norm of Fisher matrix naturally leads to this balance according to Theorem \ref{thm: diagnet-sharpness} and \eqref{eqn: 111}. In contrast, the stability of GD only controls $\max_{j\in [d]} (a_j^2+b_j^2)$, which  does not have the  balancing effect except  for the coordinate: $k\in \argmax_j (a_j^2+b_j^2)$.

\vspace*{-.4em}
\paragraph*{Deep diagonal linear networks.} Similar to \citet{nacson2022implicit},
we can analyze the interaction between depth and stability by examining the deep model $f(x;\theta)=\langle a^D\odot b^D,x\rangle$, where $a^D=(a_1^D,\dots,a_d^D)$ and $b^D$ is defined similarly. 
Analogous to Theorem \ref{thm: diagnet-sharpness} and \eqref{eqn: 111}, one can show that $\tr(G(\theta))\gtrsim \|a^D\odot b^D\|_{p}^p$ with $p=\frac{2(D-1)}{D}$.  Thus, the stability-induced regularization  changes from the $\ell_1$ norm for $D=2$ to the $\ell_2$ norm for $D\to \infty$. Here we do not discuss this in detail since it does not reveal any new insights beyond \citet[Section 6]{nacson2022implicit}.

\subsection{Numerical validations}
% We provide experiments to  support our theoretical findings.
Consider $f^*(x)=\beta_*^Tx$ with  $\beta_* = (1,1,1,0,\cdots,0)$, for which $\|\beta_*\|_1=3$. We set $d=1000, n=300$ and initialize the model by $a_j,b_j\stackrel{iid}{\sim} \cN(0,1)$ for $j=1,\dots,d$. This large initialization is adopted to eliminate the implicit regularization of small initialization.
The model is trained by SGD and GD with varying LRs. Gradient clipping is applied to stabilize the training for the case of large LR. The results are reported in Figure \ref{fig: diag-net-1}. 
% Note that we obseve that the effect of gradient clipping is  the same as the one in Figure \ref{fig: relu-net-sgd-process}

\vspace*{-.2em}
\paragraph*{The sharpness.}
The left panel of Figure \ref{fig: diag-net-a} shows how the actual sharpness of minima selected by SGD changes as increasing the LR.  
One can see that the $\tr(G(\theta))$ keeps close to  $2/\eta$--the upper bound ensured by the linear stability; $\|G(\theta)\|_F$ also decreases with LR though the decreasing is not significant. These are in contrast to $\|G(\theta)\|_2$, which keeps almost unchanged. These observations suggest that for diagonal linear networks, the linear stability is critical in characterizing the sharpness of minima found by SGD, which is consistent with the fact that in this case, SGD converges to global minima stably.
 As a comparison, the right panel of Figure \ref{fig: diag-net-b} shows that for GD, all the three sharpness keep almost unchanged when increasing the LR. 

\vspace*{-.2em}
\paragraph*{The test performance.}
Figure \ref{fig: diag-net-b} shows the test errors of minima found by SGD and GD for varying LRs. One can see that as increasing the LR, the test error of SGD decreases significantly. This can be explained by the fact that  $\tr(G(\theta))$ and the resulting $\ell_1$ norm of $\beta$ decrease significantly as demonstrated in Figure \ref{fig: diag-net-a}. In contrast, the test error of GD  barely changes, which is also consistent with our theoretical prediction that the stability of GD can not yield effective capacity control for diagonal linear networks.  These are consistent with our theoretical prediction: The stability-induced regularization of SGD is much stronger than that of GD. 

\vspace*{-.2em}
\paragraph*{The balancing effect.} 
To measure the balancedness between the inner and outer layers, we define
$
r(\theta) = {\|\alpha\|_1}/{2\|\beta\|_1} = {\sum_j(a_j^2+b_j^2)}/{(2\sum_{j}|a_jb_j|)}. 
$
By the AM-GM inequality,  $r(\theta)\geq 1$ and the equality is reached when $a_j^2=b_j^2$ for all $j\in [n]$, i.e., the solutions are totally balanced. The larger $r(\theta)$ is, the less balanced the solution is. In the experiment, we consider the initialization $a_j\sim \cN(0, 0.1), b_j\sim \cN(0, 0.1 r_0)$ with $r_0$ controlling the balancedness at initialization. We are interested in the balancedness of minima selected by SGD and GD. 
Figure \ref{fig: diag-net-c} shows that SGD finds solutions with $r(\theta)\approx 1$ no matter how unbalanced the initialization is. In contrast,  GD  is expectedly unable to reduce the unbalancedness introduced at initialization. These confirm again  our theoretical predictions by analyzing the dynamical stability.

\vspace*{-.1em}
\section{Conclusion}
In this paper, we study the stability-induced regularization of SGD and GD by relating the dynamical stability to the sharpness of local landscape. We establish generalization bounds of stable minima for two-layer ReLU networks and diagonal linear networks via linking sharpness to parameter norms. Specifically, these bounds imply that stable minima of SGD provably generalize well and can explain the benefit of  using a large LR. Most importantly, our stability analysis can explain why SGD generalizes better than GD at least for diagonal linear networks. We also corroborate our theoretical findings with fine-grained numerical experiments.

Note that the stability-induced regularization is independent of initialization but crucially depends on the size of LR. This can potentially explain the practical observation that large LR often leads to better generalization in training large-scale models. In contrast, other mechanisms  such as small initialization \cite{chizat2020implicit,woodworth2020kernel} and noise-driven diffusion \cite{blanc2020implicit,li2021happens,damian2021label} cannot explain the benefit of large LR. 
In addition, our analysis also suggests that  gradient clipping has an implicit regularization effect in the way of allowing convergence with a larger LR. We leave the systematic investigation of these issues to future work.

\subsection*{Acknowledgements}
\vspace*{-0.4em}
The work of Lei Wu is supported by a startup fund from Peking University. The work of Weijie J. Su is 
supported in part by NSF Grants CAREER DMS-1847415 and an Alfred Sloan Research Fellowship.
We thank Yaroslav Bulatov for bringing the reference \citep{defossez2015averaged} to our attention and many helpful discussions. We also thank the anonymous reviewers for their valuable suggestions. 

\bibliography{ref}
\bibliographystyle{icml2023}

%%%%%%%%%%%%%%%%%%%%%%%%%%%%%%%%%%%%%%%%%%%%%%%%%%%%%%%%%%%%%%%%%%%%%%%%%%%%%%%
%%%%%%%%%%%%%%%%%%%%%%%%%%%%%%%%%%%%%%%%%%%%%%%%%%%%%%%%%%%%%%%%%%%%%%%%%%%%%%%
% APPENDIX
%%%%%%%%%%%%%%%%%%%%%%%%%%%%%%%%%%%%%%%%%%%%%%%%%%%%%%%%%%%%%%%%%%%%%%%%%%%%%%%
%%%%%%%%%%%%%%%%%%%%%%%%%%%%%%%%%%%%%%%%%%%%%%%%%%%%%%%%%%%%%%%%%%%%%%%%%%%%%%%
\newpage
\appendix
\onecolumn

% \appendixpage
\begin{center}
    \noindent\rule{\textwidth}{4pt} \vspace{-0.2cm}
    
    \LARGE \textbf{Appendix} % \\ ~\\[-0.5cm]
    
    \noindent\rule{\textwidth}{1.2pt}
\end{center}

% \startcontents[sections]
% \printcontents[sections]{l}{1}{\setcounter{tocdepth}{2}}

\section{Technical background}
In this section, we will first introduce some notations and technical background which will be used in the proofs of next sections. 
\subsection{The Hermite expansion.}
Let $\gamma=\cN(0,1)$ and $\{h_i\}_{i=0}^{\infty}$ be the probabilist's Hermite polynomials, which form a set of orthonormal basis of $L^2(\gamma)$ with
\begin{equation}\label{eqn: hermite-function}
    h_0(z)=1,\, h_1(z)=z,\,  h_2(z)=\frac{z^2-1}{\sqrt{2}}, h_3(z)=\frac{z^3-3z}{\sqrt{6}},\cdots.
\end{equation}
Given a $f\in L^2(\gamma)$, denote by  $f(z)=\sum_k \hf_k h_k(z)$ be the Hermite expansion of $f$ where
$$
  \hf_k = \EE_{z\sim\gamma}[f(z)h_k(z)] = \frac{1}{\sqrt{2\pi}}\int_\RR f(z)h_k(z)\dd z
$$
is the ``Fourier coefficient'' of $f$. We will frequently use the following lemma \citep[Proposition 11.31]{o2014analysis}:
\begin{lemma}\label{lemma: app-hermite}
Given $f,g\in L^2(\gamma)$, we have for any $u,v\in \SS^{d-1}$ that
$$
  \EE_{x\sim\cN(0,I_d)}[f(u^Tx)g(v^Tx)] = \sum_{k=0}^\infty \hf_k \hg_k (u^Tv)^k.
$$
\end{lemma}

\subsection{Rademacher complexity and generalization bounds}
\label{sec: empirical-process}

Here we only state properties of Rademacher complexity that will be used in this paper. For the missing proofs and more details, we refer to \citet[Section 26]{shalev2014understanding}.
\begin{definition}
Given a function class $\cF$, the Rademacher complexity of $\cF$ with respect to $x_1,\dots,x_n$ is defined as 
\[
    \erad(\cF) = \EE_{\xi_1,\dots,\xi_n} [\sup_{f\in\cF} \fn\sumin f(x_i)\xi_i],
\]
where $\xi_1,\dots,\xi_n$ are \iid samples drawn from the Rademacher distribution: $\PP(\xi=1)=\PP(\xi=-1)=\half$.
\end{definition}

\begin{lemma}[Contraction property]\label{lemma: contraction}
Let $\varphi:\RR\mapsto\RR$  be $\beta$-Lispchitz continuous and $\varphi\circ\cF=\{\varphi\circ f: f\in\cF\}$. Then, 
$
    \erad(\varphi\circ\cF)\leq \beta \, \erad(\cF).
$
\end{lemma}

\begin{lemma}\label{lemma: rad-linear-class}
Let $\cF=\{u^Tx: u\in\SS^{d-1}\}$ be the linear class.  Then 
$
    \erad(\cF)\leq \sqrt{\frac{\sum_{i=1}^n \|x_i\|^2}{n^2}}.
$
\end{lemma}

\begin{theorem}\label{thm: gen-err-rademacher-complexity}
Consider a function class $\cF$  with $\sup_{z\in\cX, f\in\cF}|f(z)|\leq B$. For any $\delta\in (0,1)$,  \wp at least $1-\delta$ over the choice of $S=(z_1,z_2,\dots,z_n)$, we have,
\[
    |\frac{1}{n}\sum_{i=1}^n f(z_i) - \EE_{z}[f(z)]| \lesssim \erad(\cF) + B\sqrt{\frac{\ln(2/\delta)}{n}}.
\]
\end{theorem}

\begin{lemma}\label{eqn: Rademacher-multiply-space}
Let $\cF$ and $\cG$ be two function classes. Suppose that $\sup_{f\in \cF}\|f\|_{\infty}\leq A$ and $\sup_{g\in \cG}\|g\|_\infty\leq B$. Define $\cF*\cG=\{f(x)g(x): \cX\mapsto\RR \,:\, f\in \cF, g\in \cG\}$. Then, 
$
\erad(\cF*\cG)\leq (A+B)(\erad(\cF)+\erad(\cG)).
$
\end{lemma}
\begin{proof}
By the definition of Rademacher complexity, 
\begin{align*}
n \erad(\cF*\cG) &= \EE_{\xi}[\sup_{f\in \cF, g\in \cG} \sumin f(x_i)g(x_i)\xi_i]\\
&=\EE_{\xi}[\sup_{f\in \cF, g\in \cG} \sumin \frac{(f(x_i)+g(x_i))^2}{4}\xi_i - \sumin \frac{(f(x_i)-g(x_i))^2}{4}\xi_i]\\
&\leq \EE_{\xi}[\sup_{f\in \cF, g\in \cG} \sumin \frac{(f(x_i)+g(x_i))^2}{4}\xi_i +\EE_{\xi}[\sup_{f\in \cF, g\in \cG} \sumin \frac{(f(x_i)-g(x_i))^2}{4}\xi_i]\\
&\stackrel{(i)}{\leq} \frac{A+B}{2}\left( \EE_{\xi}[\sup_{f\in \cF, g\in \cG} \sumin (f(x_i)+g(x_i))\xi_i] +\EE_{\xi}[\sup_{f\in \cF, g\in \cG} \sumin (f(x_i)-g(x_i))\xi_i]\right)\\
&\leq (A+B)n(\erad(\cF)+\erad(\cG)),
\end{align*}
where $(i)$ follows from the Lemma \ref{lemma: contraction} and the fact that $t^2/4$ is $(A+B)/2$ Lipschitz continuous since $|f|\leq A, |g|\leq B$.
\end{proof}

\paragraph*{Generalization bounds of learning with a smooth loss.}
Let $\phi: \cY\times \cY \mapsto[0,\infty)$ be a loss function. Define the empirical and population risk as follows
\[
\erisk(h) = \hEE[\phi(h(x),y)],\quad \risk(h) = \EE[\phi(h(x),y)],
\] 
where $\hEE$ denotes the expectation with respect to the empirical measure. 
Let $\cH$ be the hypothesis space and $\hh = \argmin_{h\in\cH} \hL(h)$. We would like to bound the population risk of the $\hh$ by using the following decomposition:
\[
    \risk(\hh) = \erisk(\hh) + \underbrace{\risk(\hh) - \erisk(\hh)}_{gen-gap}.
\]
Theorem \ref{thm: gen-err-rademacher-complexity} shows that the second term (gen-gap) can be controlled by the Rademacher complexity of $\cH$. By assuming $\phi$ is Lipschitz continuous and applying Lemma \ref{lemma: contraction}, an (informal) bound goes like 
\[
\risk(\hh) - \erisk(\hh)\leq \sup_{h\in\cH}|\risk(h)-\erisk(h)|\leq Lip(\phi)\erad(\cH).
\]
This usually provides us a $O(1/\sqrt{n})$ bound, which is  tight for Lipschitz loss such as hinge loss. However, for square loss, this bound is often loose as explained as follows. For the minimizer $\hh$, it is expected that $\risk(\hh)\leq r$ for small $r$. Therefore, one only needs to consider a constraint hypothesis class:
\[
    \cH_r=\{ h\in \cH \,|\, \risk(h)\leq r \}.
\]
For hypothesis in this restricted class, the Lipschitz constant of $\phi$ is much smaller for smooth loss. For instance, $t^2/2$ is only $r$-Lipschitz for $t\in [-r,r]$. This argument can be formalized by using the concept of local Rademacher complexity \citep{bartlett2005local}.  Specifically, we shall use the following theorem in our proof, which is a restatement of \citet[Theorem 1]{srebro2010smoothness}
\begin{theorem}\label{thm: fast-rates-smoothloss}
Let 
\begin{equation}\label{eqn: worst-case-rad}
\mathfrak{R}_n(\cH)=\sup_{x_1,\dots,x_n}\erad(\cH)
\end{equation}
be the worst-case Rademacher complexity. 
Assume that $|\phi''|\leq A$ and $0\leq \phi\leq B$.. Then, \wp at least $1-\delta$ over the sampling of training set, we have for any $h\in \cH$ that 
\[
\risk(h)\leq \erisk(h) + C \left(\sqrt{\erisk(h)}\left(\sqrt{A}\log^{3/2}(n) \mathfrak{R}_n(\cH)+\sqrt{\frac{B\log(1/\delta)}{n}}\right)+ A \log^{3}(n) \mathfrak{R}_n(\cH)^2 + \frac{B\log(1/\delta)}{n}\right).
\]
In particular, for $\hh\in \argmin_h \erisk(h)$, 
\begin{equation}\label{eqn: risk-bound-minimizer}
\risk(\hh)\lesssim  A \log^{3}(n) \mathfrak{R}_n(\cH)^2 + \frac{B\log(1/\delta)}{n}.
\end{equation}
\end{theorem}
In this paper, we will mainly use  \eqref{eqn: risk-bound-minimizer} to bound  generalization error since our focus is the minimizer $\hh$.

\subsection{Concentration inequalities}
\label{app: concentration}

\begin{definition}\label{definition: orlic-norm}
Let $\psi$ be a non-decreasing, convex function with $\psi(0)=0$. 
The Orlicz norm of a random variable $X$ is defined by 
$
\|X\|_{\psi}:=\inf\{t>0: \EE[\psi(|X|/t)]\leq 1\}.
$
If $X\in\RR^d$ is a vector, then $\|X\|_{\psi}:=\sup_{u\in \SS^{d-1}}\|u^TX\|_{\psi}$.
\end{definition}
For our purpose, Orlicz norms of interest are the ones given by $\psi_p(x)=e^{x^p}-1$ for $p\geq 1$. In particular, the cases of $p=1$ and $p=2$ correspond to the sub-exponential and sub-Gaussian norms, respectively. A random variable $X$ is said to be sub-Gaussian (resp. sub-exponential) if $\|X\|_{\psi_2}<\infty$ (resp. $\|X\|_{\psi_1}<1$).

A random variable with finite $\psi_p$-norm has the following  control of the tail behavior
\[
    \PP\{|X|\geq t\}\leq C_1 e^{-C_2 \frac{t^p}{\|X\|^p_{\psi_p}}},
\]
where $C_1,C_2$ are constant that only depend on  $p$.

\begin{lemma}\label{lemma: pro-orlic-norm}
\begin{itemize}
    \item If $|X|\lesssim 1$ almost surely, then $\|X\|_{\psi_i}\lesssim 1$ for $i=1,2$.
    \item If $X\sim \cN(0,\sigma^2)$, $X$ is sub-Gaussian with $\|X\|_{\psi_2}\leq C\sigma$.
\item Let $X,Y$ be sub-Gaussian random variables. Then, $XY$ is sub-exponential and
$
    \|XY\|_{\psi_1}\leq \|X\|_{\psi_2}\|Y\|_{\psi_2}.
$
\item If $|X|\leq |Y|$ a.s., then $\|X\|_{\psi}\leq \|Y\|_{\psi}$ for any $\psi$ that satisfies the condition in Definition \ref{definition: orlic-norm}.
\item \textbf{Center inequality.} For a random variable $X$, we have 
\begin{equation}\label{eqn: centering-inequality}
\|X-\EE[X]\|_{\psi_p}\leq C\|X\|_{\psi_p} 
\end{equation}
for a constant $C>0$ that may depend on $p$.
\end{itemize}
\end{lemma}

\begin{theorem}[Bernstein's inequality]\label{thm: bernstein}
Let $X_1,\dots,X_n$ be independent sub-exponential random variables. Suppose $K=\max_{i}\|X_i\|_{\psi_1}<\infty$. Then, for any $t>0$, 
\[
\PP\Big\{\big|\frac{1}{n}\sum_{i=1}^n X_i - \EE[X]\big|\geq t\Big\}\leq 2 \exp\left(- C n \min\left(\frac{t^2}{K^2}, \frac{t}{K}\right)\right).
\]
\end{theorem}

\begin{proposition}[Sums of independent sub-Gaussians]\label{pro: sum-sub-Gaussians}
Let $X_1,\dots,X_n$ be independent, mean zero, sub-Gaussian random variables. Then, $\sum_{i=1}^n X_i$ is also a sub-Gaussian random variable, and 
\[
    \|\sum_{i=1}^n X_i \|_{\psi_2}^2 \leq C \sum_{i=1}^n \|X_i\|_{\psi_2}^2.
\]
\end{proposition}

\paragraph*{Covering number.} We shall also use the covering number in our analysis.
Let $(T,q)$ be a metric space. Consider a subset $K\subset T$ and let $\varepsilon>0$. A subset $\cN_\varepsilon$ is called an $\varepsilon$-net of $K$ if every point in $K$ is within a distance $\varepsilon$ of some point of $\cN_\varepsilon$, i.e., 
\[
\forall x\in K, \,\exists\, x_0\in \cN_\varepsilon\,: q(x,x_0)\leq \varepsilon.
\]
The smallest possible cardinality of an $\varepsilon$-net of $K$ is called the covering number of $K$ and is denoted by $N(K, q,\varepsilon)$.

A commonly-used fact is   
\begin{equation}\label{eqn: covering-number-sphere}
N(\SS^{d-1}, \|\cdot\|,\varepsilon)\leq (1+2/\varepsilon)^d
\end{equation}
(see, e.g., \citet[Corollary 4.2.13]{vershynin2018high}).

\paragraph*{Remark:}
We refer the reader to \citet{vershynin2018high}  for the proofs of the above properties and more related information.

\subsection{Auxiliary Lemmas}
\begin{lemma}\label{lemma: schur-product}
Let $u_1,u_2,\dots,u_m\in\RR^d$. Then for any $k\in\mathbb{N}$ and $\alpha\in\RR^m$, we have 
\begin{equation}\label{eqn: app-x1}
  \sum_{i,j=1}^m \alpha_i \alpha_j (u_i^Tu_j)^k \geq 0.
\end{equation}
\end{lemma}
\begin{proof}
Let $U=(u_1,\dots,u_m)\in\RR^{d\times m}$ and $Q_k=((u_i^Tu_j)^k)_{i,j} \in\RR^{m\times m}$. First, $Q_0=I_d, Q_1=U^TU$ are both positive semi-definite and hence \eqref{eqn: app-x1} holds. For $k\geq 2$, we have $Q_k=Q_1\circ Q_1\circ \cdots \circ Q_1$ where $\circ$ denotes the hadamard product. By the Schur product theorem\footnote{see \url{https://en.wikipedia.org/wiki/Schur_product_theorem}}, $Q_k$ is also positive semi-definite and hence \eqref{eqn: app-x1} holds.
\end{proof}

\begin{lemma}\label{lemma: variation-principle}
Suppose  $k(\cdot,\cdot)$ to be positive semi-definite kernel and let  $\phi: \cX\mapsto \cH$ be a feature map satisfying $k(x,y)=\<\phi(x), \phi(y)\>_{\cH}$. Then, 
\begin{equation}
\lambda_1(\cK) = \sup_{\|h\|_{\cH}= 1} \EE_{x}[\< h, \phi(x)\>_{\cH}^2].
\end{equation}
\end{lemma}
\begin{proof}
By the variational principle of the largest eigenvalue, we have 
\begin{align*}
\lambda_1(\cK) &= \sup_{\|u\|_{L_2(\rho)}=1} \EE_{x,y}[k(x,y)u(x)u(y)] = \sup_{\|u\|_{L_2(\rho)}=1} \EE_{x,y}[\< \phi(x), \phi(y)\>_{\cH}u(x)u(y)]\\
&= \sup_{\|u\|_{L_2(\rho)}=1}\|\EE_x[u(x)\phi(x)]\|_{\cH}^2 = \sup_{\|u\|_{L_2(\rho)}=1}\sup_{\|h\|_{\cH}=1}\langle h, \EE_x[u(x)\phi(x)]\rangle^2_{\cH}\\
&= \sup_{\|h\|_{\cH}=1}\sup_{\|u\|_{L_2(\rho)}=1}\EE_{x}[u(x)\langle h, \phi(x) \rangle_{\cH}]^2= \sup_{\|h\|_{\cH}=1}\EE_{x}[\langle h, \phi(x) \rangle_{\cH}^2].
\end{align*}
\end{proof}

\begin{lemma}\label{lemma: covariance-gaussian}
Assume $X\sim\cN(0,I_d)$. For any $\delta\in (0,1)$, let $n\gtrsim d+\log(1/\delta)$, then \wp $1-\delta$, we have 
\[
  \|\hSigma_n-\Sigma\|_2 \lesssim \sqrt{\frac{d+\log(1/\delta)}{n}}+\frac{d+\log(1/\delta)}{n}, \qquad \|\hSigma_n\|_2\lesssim 1,\qquad \|\hSigma_n\|_F\lesssim \sqrt{d}.
\]
\end{lemma}
\begin{proof}
First by \citet[Exercise 4.7.4]{vershynin2018high}, for any $\delta\in (0,1)$, \wp $1-\delta$ it holds that
\[
  \|\hSigma_n-\Sigma\|_2\lesssim \left(\sqrt{\frac{d+\log(1/\delta)}{n}}+\frac{d+\log(1/\delta)}{n}\right)\|\Sigma\|_2.
\]
Therefore, if $n\gtrsim d+\log(1/\delta)$, we have 
\[
  \|\hSigma_n\|_2\leq \|\hSigma_n-\Sigma\|_2 + \|\Sigma\|_2\lesssim \|\Sigma\|_2=1.
\]
Moreover, 
\begin{align*}
\|\hSigma_n\|^2_F = \sum_{j=1}^d \lambda_j^2(\hSigma_n)\leq d \|\hSigma_n\|_2\lesssim d.
\end{align*}
Taking the square root completes the proof.

% For the Frobenius norm, we have 
% \begin{align*}
% \|\hSigma_n\|_F &\leq \|\hSigma_n-\Sigma\|_F + \|\Sigma\|_F \leq \sqrt{\sum_{j,k=1}^d |(\hSigma_n)_{j,k}-\Sigma_{j,k}|^2} + d\\
% &\leq \sqrt{d^2 \frac{\log(d/\delta)}{n}} + d\lesssim d,
% \end{align*}
% where the last step uses the assumption $n\gtrsim \log(d/\delta)$.
\end{proof}
\section{Missing Proofs of Section \ref{sec: stability}}

\subsection{Proof of Proposition \ref{pro: stability-1}.}
\label{sec: proof-stability-1}

Let $\cS_{+}$ be the set of $p\times p$ positive semi-definite matrices, $H_i=g_i(\theta^*)g_i(\theta^*)^T$ for $i\in [n]$ and $\delta_t = \theta_t-\theta^*$. 
The linearized SGD \eqref{eqn: linear-sgd} can be rewritten as
$
    \delta_{t+1} =(I-\eta H_{i_t})\delta_t
$
with $i_t\in \unif([n])$.
Then, 
\begin{align*}
    \delta_{t+1}\delta_{t+1}^T &= (\delta_t - \eta H_{i_t} \delta_t)(\delta_t - \eta H_{i_t} \delta_t)^T\\
    &= \delta_t \delta_t^T - \eta (\delta_t\delta_t^T H_{i_t} + H_{i_t} \delta_t\delta_t^T) + \eta^2 H_{i_t}\delta_t\delta_t^TH_{i_t}.
\end{align*}
Let $Q_t = \EE[\delta_t\delta_t^T]$ be the deviation covariance matrix. Then taking expectation gives 
\begin{align}\label{eqn: 55}
\notag Q_{t+1} &= Q_t - \eta (Q_t H+ HQ_t) + \eta^2 \EE[H_{i_t} Q_t H_{i_t}]\\
 &= (I-\eta T_\eta) Q_t,
\end{align}
where $T_\eta: \cS_{+}\mapsto\cS_{+}$ is given by 
$
    T_\eta A = (HA+AH) - \eta \EE[H_{\xi} AH_\xi],
$
where the expectation is taken with respect to $\xi\sim \text{Unif}([n])$. 

By \eqref{eqn: 55}, to ensure $\|\EE[Q_t]\|_F\leq C\|\EE[Q_0]\|_F$ for some constant $C>0$, we need  $T_\eta\succeq 0$. This is equivalent to it holds that
\begin{equation}\label{eqn: 56}
\langle A, T_\eta A\rangle  = 2\tr(AHA) - \eta \EE[\tr(AH_\xi)AH_\xi] \geq 0 \quad \forall \, A\in \cS_{+}.
\end{equation}
Noticing that  
$
\EE[\tr(AH_\xi A H_\xi)] = \EE[(g_\xi^TAg_\xi)^2]\geq (\EE[g_\xi Ag_\xi^T])^2 = \tr^2(HA),
$
\eqref{eqn: 56} implies 
\[
    2\tr(HA^2) - \eta \tr^2(AH)\geq 0,\quad \forall A\in \SS_{+}.
\]
Taking $A=\diag(w_1,\dots,w_p)$, we obtain 
\begin{equation}\label{eqn: 57}
    \frac{(\sum_j \lambda_j(H)w_j)^2}{\sum_{j}\lambda_j(H) w_j^2}\leq \frac{2}{\eta}.
\end{equation}
Specifically, taking $w_j=1$ for $j=1,\dots, n$ completes the proof.
\qed

\begin{remark}
It should be stressed that the stability condition \eqref{eqn: 57} is stronger than $\tr(H)\leq 2/\eta$.  We only state the latter in the main text since it is more intuitive and has clean relationship to sharpness.
\end{remark}

\subsection{ Proof of Proposition \ref{pro: stability-2}}
\label{sec: proof-stability-2}
By \eqref{eqn: hessian-expression}, we have $H(\theta) = G(\theta) + O(\varepsilon^{1/2})$ for any $\theta\in \cQ_{\varepsilon, \eta}$.  Then, we have 
\begin{equation}\label{eqn: 58}
\tr[H(\theta_t)S(\theta_t)] = \tr[G(\theta_t)S(\theta_t)] + O(\varepsilon^{1/2})\geq \mu_0 \|G(\theta_t)\|_F^2 + O(\varepsilon^{1/2}),
\end{equation}
where the second step follows from the definition of $\mu(\theta)$ and the assumption that $\mu(\theta)\geq \mu_0$. 

Let $\gamma:=\eta^2 \inf_{\theta\in \cQ_{\varepsilon,\eta}}\mu_0 \|G(\theta)\|_F^2$.
Then combining Lemma \ref{lemma: loss-update} and \eqref{eqn: 58} gives 
\begin{align*}
\EE[\erisk(\theta_t)]&\geq \eta^2 \mu_0 \|G(\theta_t)\|_F^2 \EE[\erisk(\theta_t)] + O(\eta^3+\eta^2\varepsilon^{3/2}) \\ 
& \geq \gamma \EE[\erisk(\theta_t)] + O(\eta^3+\eta^2\varepsilon^{3/2})\\ 
&\geq \gamma^t \EE[\erisk(\theta_0)] + \frac{\gamma^t-1}{\gamma-1} O(\eta^3+\eta^2\varepsilon^{3/2}).
\end{align*}
\qed

% \begin{corollary}
% Suppose $f^*(x)=\beta_*^Tx$ with $w_*$ being $k$-sparse. Assume $\htheta$ is a global minimum that is linearly stable for SGD \eqref{eqn: mini-batch-sgd} with learning rate $\eta$. Then,  for any $\delta\in (0,1)$, if $n\gtrsim \log(d/\delta)$, we have \wp $1-\delta$ that
% \[
%   \risk(\htheta)\leq \frac{k\log(d)}{n}.
% \]
% \end{corollary}

\section{Missing Proofs in Section \ref{sec: relu-net}}
\label{sec: proof-2lnn}

% Then simple calculations give  
% \begin{equation} \label{eqn: pop-sharpness-2lnn}
% \begin{aligned}
% \tr(G(\theta)) &= \sumjm \left(\varphi_1(w_j,w_j) + d a_j^2 \varphi_2(w_j,w_j)\right)\\
%  \|G(\theta)\|_F^2 &= \sum_{j,k} \|F_{j,k}\|_F^2 \geq \sum_{j,k}\left(\varphi_1(w_j,w_k)^2  + a_j^2a_k^2A_{j,k}\right),
% \end{aligned}
% \end{equation}
% where $A_{j,k}=\| \EE[\sigma'(w_j^Tx)\sigma'(w_k^Tx)xx^T]\|_F^2$. 

In this section, we will frequently use the following definition and results.
\begin{itemize}
    \item \textbf{Kernel functions.} Define two associated kernel functions:
\begin{equation}
\varphi_1(u,v):=\EE_x[\sigma(u^Tx)\sigma(v^Tx)],\quad 
\varphi_2(u,v):=\EE_x[\sigma'(u^Tx)\sigma'(v^Tx)],
\end{equation}
where $\varphi_1,\varphi_2: \Omega \mapsto\RR$ with $\Omega:=\SS^{d-1}\otimes \SS^{d-1}$.
The corresponding empirical ones are given by
\begin{equation}
\hph_1(u,v) = \frac{1}{n}\sumin \sigma(u^Tx_i)\sigma(v^Tx_i),\quad \hph_2(u,v) = \fn\sumin \sigma'(u^Tx_i)\sigma'(v^Tx_i).
\end{equation}

\item \textbf{Hermite expansions of kernels.} Let $\sigma(t)=\sum_{k=0}^\infty \alpha_k h_k(t)$ and  $\sigma'(t)=\sum_k \beta_k h_k(t)$ be the Hermite expansions of $\sigma$ and $\sigma'$, respectively.
\end{itemize}

\begin{lemma}\label{lemma: 2sphere}
 Define $\|(u,v)-(u',v')\|_\Omega=\|u-u'\|+\|v-v'\|$ for any $(u,v),(u',v')\in\Omega$. Then, 
\begin{equation}\label{eqn: cover-2sphere}
\cN(\Omega, \|\cdot\|_\Omega, \epsilon)\leq \cN(\SS^{d-1},\|\cdot\|, \epsilon/2)^2\leq (6/\epsilon)^{2d}.
\end{equation}
\end{lemma}
\begin{proof}
Follow trivially from the fact \eqref{eqn: covering-number-sphere}.
\end{proof}

\begin{lemma}[Property of kernel functions]\label{lemma: kernel-property}
\begin{itemize}
\item $\varphi_1(u,v)=\|u\|\|v\|\sum_k \alpha_k^2 (\hu^T\hv)^k, \varphi_2(u,v)=\sum_{k}\beta_k^2 (\hu^T\hv)^k$.
\item  $\alpha_0=\beta_1\sim 1$ and $\beta_0\sim 1$.
\item $\varphi_i(u,v)\sim 1$ for any $u,v\in\SS^{d-1}$ and $i=1,2$.
\item For any $s\in\RR^m$ and $u_1,u_2,\dots,u_m\in\SS^{d-1}$, we have 
\[
    \sum_{j,k=1}^m s_j^2s_k^2 \varphi_i(u_j^Tu_k)\gtrsim (\sum_{j=1}^m s_j^2)^2,\quad \forall i=1,2.
\]
\end{itemize}
\end{lemma} 

\begin{proof}
\begin{itemize}
\item Using the positive homogeneity  of $\sigma$ and  Lemma \ref{lemma: app-hermite}, we have
\begin{align}\label{eqn: 77}
\notag \varphi_1(u,v) &= \|u\|\|v\|\EE_{x}[\sigma(\hu^Tx)\sigma(\hv^Tx)]= \|u\|\|v\|\EE_{x}[\sum_k \alpha_k h_k(\hu^Tx)\sum_l \alpha_l h_l(\hv^Tx)]\\
&=  \|u\|\|v\| \sum_{k,l}\alpha_k \alpha_l \delta_{k,l}(\hu^T\hv)^k=\|u\|\|v\| \sum_{k}\alpha_k^2(\hu^T\hv)^k.
\end{align}
Similarly, we have the expansion of $\varphi_2$.
\item Noticing $h_0(z)=1$ and $h_1(z)=z$, we have
\begin{align*}
  \alpha_0 &= \frac{1}{\sqrt{2\pi}}\int \sigma(t)e^{-t^2/2}\dd t\sim \int_0^\infty te^{-t^2/2}\dd t\sim 1\\
  \beta_0& = \frac{1}{\sqrt{2\pi}}\int \sigma'(t) e^{-t^2/2}\dd t \sim  \frac{1}{\sqrt{2\pi}}\int_0^\infty e^{-t^2/2}\dd t\sim 1\\ 
  \beta_1 &= \frac{1}{\sqrt{2\pi}}\int \sigma'(t)t e^{-t^2/2}\dd t = \frac{1}{\sqrt{2\pi}}\int \sigma(t) e^{-t^2/2}\dd t=\alpha_0\sim 1.
\end{align*}
\item We now prove the third conclusion. By \eqref{eqn: 77}, we have
\begin{align*}
\sum_{j,k=1}^m s_j^2s_k^2 \varphi_1(u_j^Tu_k) &= \sum_{j,k=1}^m s_j^2s_k^2 \sum_{l=0}^\infty \alpha_l^2 (u_j^Tu_k)^l = \alpha_0^2 \sum_{j,k=1}^m s_j^2s_k^2 + \sum_{l=1}^\infty \alpha_l^2\sum_{j,k=1}^m s_j^2 s_k^2 (u_j^Tu_k)^l\\ 
&\geq \alpha_0^2 \sum_{j,k=1}^m s_j^2s_k^2 + 0 \qquad \text{(use Lemma \ref{lemma: schur-product})}\\ 
&\gtrsim (\sum_j s_j^2)^2,
\end{align*}
where the last step is due to $\alpha_0\sim 1$. The case of $i=2$ can be proved analogously.
\end{itemize}

\end{proof}

\subsection{The kernel concentrations.}
Before proving the equivalence between  sharpness and parameter norms, we first need to bound the difference between the poluation kernels and empirical kernels.
 % Lemma \ref{lemma: empirical-kernel}, which establishes the closeness between population and empirical kernels. Hence, one can intuitively do not distinguish them too much. 

\begin{lemma}\label{lemma: covering-concentration}
Let  $\cH=\{h(\cdot;\theta): \theta\in \Theta\}$.
Denote by $\omega: [0,\infty)\mapsto [0,\infty)$ a modulus of continuity of $h$ in the sense $\sup_{x\in\cX}|h(x;\theta_1)-h(x;\theta_2)|\leq \omega(\|\theta_1-\theta_2\|)$. Suppose that $\forall\, \theta\in \Theta$, $h(X;\theta)$ is mean zero and sub-exponential with $\|h(X;\theta)\|_{\psi_1}\leq K$. Let $N_{\varepsilon}$ be the covering number  of $\Theta$ with respect to $\|\cdot\|$.  
Then for any $\delta\in (0,1)$,  \wp at least $1-\delta$ it holds that
\[
\sup_{\theta\in\Theta} \left|\fn \sumin h(x_i;\theta)\right| \lesssim \omega(\varepsilon) + K \max\left(\frac{\log(N_\varepsilon/\delta)}{n}, \sqrt{\frac{\log(N_\varepsilon/\delta)}{n}}\right).
\]
\end{lemma}
\begin{proof}
Let $\Theta_\varepsilon$ be an $\varepsilon$-cover of $\Theta$. For any $\theta\in \Theta$, let $\theta'$ be an element in $\Theta_\varepsilon$ such that $\|\theta-\theta'\|\leq \varepsilon$.
Then,
\begin{align}\label{eqn: app-123}
\notag  \sup_{\theta\in \Theta} \big|\fn \sumin h(x_i;\theta)\big| &=\sup_{\theta\in \Theta} \left|\fn \sumin \left(h(x_i;\theta)-h(x_i;\theta')+h(x_i;\theta')\right)\right|\\
  &\leq \omega(\varepsilon) + \sup_{\theta\in \Theta_\varepsilon} \left|\fn \sumin h(x_i;\theta)\right|.
\end{align}
Since $h(X;\theta)$ is sub-exponential with $\|h(X;\theta)\|_{\psi_1}\leq K$. By the Bernstein inequality (Theorem \ref{thm: bernstein}), it holds for any $\theta\in \Theta_\varepsilon$ that
\[
\PP\left\{\left|\fn \sumin h(x_i;\theta)\right|\geq t\right\}\leq 2 e^{-C n \min(\frac{t}{K},\frac{t^2}{K^2})}.
\]
Taking the union bound over $\Theta_{\varepsilon}$ leads to
\[
\PP\left\{\sup_{\theta\in \Theta_\varepsilon} \left|\fn \sumin h(x_i;\theta)\right|>t\right\}\leq |\Theta_{\varepsilon}| 2e^{-C n \min(\frac{t}{K},\frac{t^2}{K^2})}.
\]
Since $|\Theta_\epsilon|=N_\varepsilon$, the above implies \wp $1-\delta$ that
\[
  \sup_{\theta\in \Theta_\varepsilon} |\fn \sumin h(x_i;\theta)|\lesssim K \max\left(\frac{\log(N_\varepsilon/\delta)}{n}, \sqrt{\frac{\log(N_\varepsilon/\delta)}{n}}\right).
\]
Substituting it into \eqref{eqn: app-123} completes the proof.
\end{proof}

\begin{lemma}\label{lemma: lip-subgaussian}
Let $Z$ be a mean-zero and sub-Gaussian random variable. Assume that $q$ is $L_q$-Lipschitz and $q(c)=0$ for some $c$. Then, $q(Z)$ is also sub-Gaussian with $\|q(Z)\|_{\psi_2}\lesssim L_q(\|Z\|_{\psi_2}+|c|)$.
\end{lemma}
\begin{proof}
By the property of sub-Gaussian random variable, we have $\EE[e^{Z^2/\|Z\|_{\psi_2}^2}]\leq 2$. Then, 
\begin{align*}
\EE[e^{q(Z)^2/(L_q^2\|Z-c\|_{\Psi_2}^2)}] &= \EE[e^{(q(Z)-q(c))^2/(L_q^2\|Z-c\|_{\Psi_2}^2)}]\\
&\leq \EE[e^{L_q^2|Z-c|^2/(L_q^2\|Z-c\|_{\Psi_2}^2)}] = \EE[e^{|Z-c|^2/\|Z-c\|_{\Psi_2}^2}]\leq 2.
\end{align*}
Hence, we have  $\|q(Z)\|_{\psi_2}\leq L_q\|Z-c\|_{\psi_2}\lesssim L_q(\|Z\|_{\psi_2}+|c|)$, where the last inequality is due to  that $\|\cdot\|_{\psi_2}$ is a norm.
\end{proof}

\begin{lemma}\label{lemma: sub-exp-activated-rv}
Let $p:\RR\mapsto\RR$ be $L_p$-Lipschitz and $p(0)=0$ and $q: \RR\mapsto\RR$ be $L_q$-Lipschitz and $q(0)=0$. Let $Z_{u,v} = p(u^TX)q(v^TX)-\EE[p(u^TX)q(v^TX)]$. Then, for any $u,v\in\SS^{d-1}$, $Z_{u,v}(X)$ is a mean-zero and satisfies 
\begin{align*}
\|Z_{u,v}(X)\|_{\psi_1}&\lesssim L_p L_q\|X\|_{\psi_2}^2\\
|Z_{u,v}(X)-Z_{u',v'}(X)|&\lesssim L_pL_q\|X\|^2(\|u-u'\|+\|v-v'\|).
\end{align*}
\end{lemma}
\begin{proof}
We first have
\begin{align*}
\|Z_{u,v}(X)\|_{\psi_1}&\leq C \|p(u^TX)q(v^TX)\|_{\psi_1}\leq C\|p(u^TX)\|_{\psi_2}\|q(v^TX)\|_{\psi_2}\\
&\leq CL_pL_q\|u^TX\|_{\psi_2} \|v^TX\|_{\psi_2}\leq CL_pL_q\|X\|_{\psi_2}^2,
\end{align*}
where the first step follows from the centering inequality \eqref{eqn: centering-inequality}; the second step is due to Lemma \ref{lemma: pro-orlic-norm}; the third inequality follows from Lemma \ref{lemma: lip-subgaussian}.

Let $J_{u,v}=p(u^TX)q(v^TX)$. Then, 
\begin{align*}
|J_{u,v}-J_{u',v'}| &= |J_{u,v}-J_{u,v'}| + |J_{u,v'}-J_{u',v'}|\\
&=|p(u^TX)||q(v^TX)-q(v'\cdot X)| + |q(v'\cdot X)||p(u^TX)-p(u'\cdot X)|\\
&\leq C L_p L_q\|X\|^2(\|v-v'\|+\|u-u'\|).
\end{align*}
Analogously, we can prove that $Z_{u,v}(X)$ satisfies the same Lipschitz condition.
\end{proof}

Now we are ready to bound the difference between population kernels and the corresponding empirical kernels.
\begin{lemma}\label{lemma: empirical-kernel}
Suppose $n\gtrsim d$.
For any $\delta\in (0,1)$,  \wp at least $1-\delta$, it holds for any $i=1,2$ that
\[
  \sup_{u,v\in\SS^{d-1}}|\varphi_i(u^Tv)  -\hph_i(u,v)| \lesssim  \min\left\{\sqrt{\frac{d\log(n/\delta)}{n}}, \frac{d\log(n/\delta)}{n}\right\}=:r_n.
\]
\end{lemma}
\begin{proof}
Recall $\Omega=\SS^{d-1}\otimes \SS^{d-1}$ and let $\|(u,v)-(u',v')\|_{\Omega} = \|u-u'\| + \|v-v'\|$ for any $(u,v),(u',v')\in \Omega$. 

\underline{\textbf{The case of $\varphi_1$.} }
Let $h(x;\theta)=\sigma(u^Tx)\sigma(v^Tx)$.
By Lemma \ref{lemma: sub-exp-activated-rv} and noticing that $\sigma$ is Lipschitz continuous, we have 
\begin{equation}\label{eqn: 60}
\begin{aligned}
\|h(X;\theta)\|_{\psi_1}&\lesssim 1\\ 
   | h(X;\theta_1) - h(X;\theta_2)|&\lesssim \|X\|^2 \|\theta_1-\theta_2\|_{\Omega} \lesssim d \|\theta_1-\theta_2\|_{\Omega}.
\end{aligned}
\end{equation}
By Lemma \ref{lemma: 2sphere},  $N_\epsilon \leq (6/\epsilon)^{2d}$. Then applying Lemma \ref{lemma: covering-concentration} gives \wp at least $1-\delta$ it holds that 
\begin{align*}
    \sup_{u,v\in \SS^{d-1}}|\hph_1(u,v) - \varphi_1(u,v)|&\lesssim d\epsilon + \max\left\{\frac{\log (N_\epsilon/\delta)}{n}, \sqrt{\frac{\log(N_\epsilon/\delta)}{n}}\right\} \\ 
    &\lesssim \frac{d}{n} + \max\left\{\frac{d\log(n/\delta)}{n},\sqrt{\frac{d\log(n/\delta)}{n}}\right\}\leq 2r_n,
\end{align*}
where we take $\epsilon=1/n$.

\underline{\textbf{The case of $\varphi_2$.}} For $\varphi_2$, the major challenge  comes the discontinuity of $\sigma'(\cdot)$. Fortunately, the concentration is still possible since $\sigma'$ is discontinuous only at the origin. Note that $\sigma'(\cdot)$ is exactly the Heaviside step function and hence, we will write $H(t)=\sigma'(t)$ for convenience. 
\begin{itemize}
\item \textbf{Step 1: smoothing the kernels.} Define two smoothed Heaviside step functions:
\[
  H_{\beta}^{-}(t) = \begin{cases}
  1 & \text{ if } t\geq \beta \\
 \frac{1}{\beta} t  &\text{ if } 0 \leq t\leq \beta\\
  0 & \text{ if } t< 0.
  \end{cases}
  \qquad 
    H_{\beta}^{+}(t) = \begin{cases}
  1 & \text{ if } t\geq 0 \\
  \frac{1}{\beta} t +1 &\text{ if } -\beta \leq t\leq 0\\
  0 & \text{ if } t< -\beta,
  \end{cases}
\]
where $\beta\in \RR_{+}$ control the degree of smoothing. 
Then, we have that $H_{\beta}^{+}, H_{\beta}^{-1}$  are both $\frac{1}{\beta}$-Lipschitz and $0\leq H_{\beta}^{-}(t)\leq H(t)\leq H_{\beta}^{+}(t)\leq 1$. An illustration of these three functions are provided in Figure \ref{fig: app-smooth-relu}.
\begin{figure}[!h]
\centering
\includegraphics[width=0.28\textwidth]{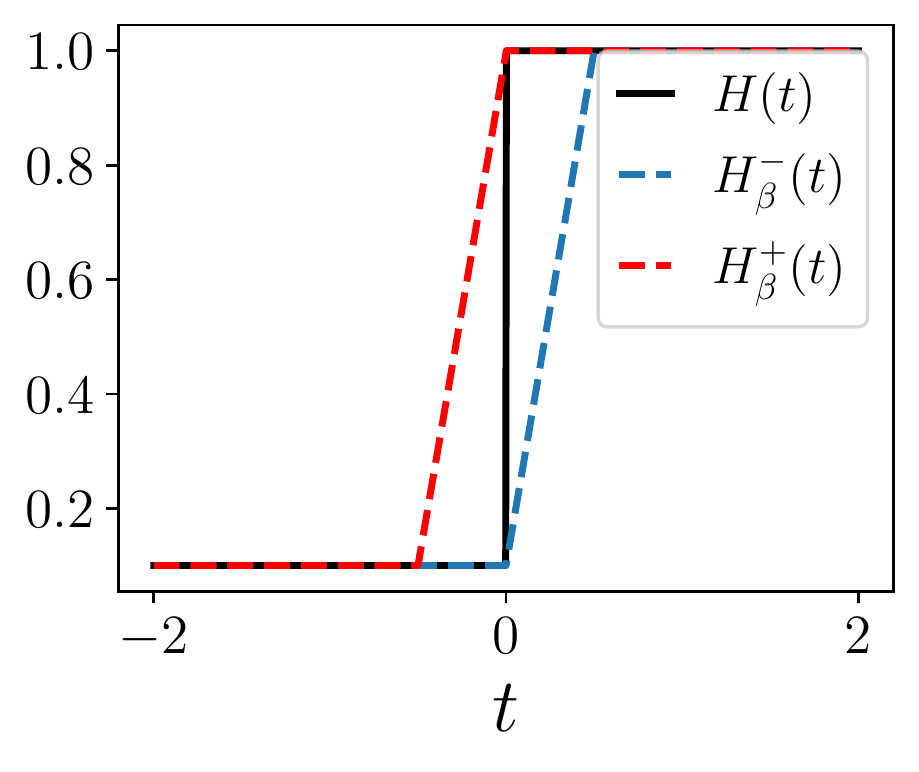}
\vspace*{-1em}
\caption{\small An visual comparison among $H,H_{\beta}^{+}$, and $H_{\beta}^{-}$.}
\label{fig: app-smooth-relu}
\end{figure}

Correspondingly, define 
\[
  \varphi_{2,\beta}^{+}(u^Tv) = \EE_{x}[H_{\beta}^{+}(u^Tx)H_{\beta}^{+}(v^Tx)],\quad \varphi_{2,\beta}^{-}(u^Tv) = \EE_{x}[H_{\beta}^{-}(u^Tx)H_{\beta}^{-}(v^Tx)].
\]
We first have 
\begin{align*}
|\varphi_{2,\beta}^{+}(u^Tv) - \varphi_{2}(u^Tv)| &= |\EE_{x}[H_{\beta}^{+}(u^Tx)H_{\beta}^{+}(v^Tx)] - \EE_{x}[H(u^Tx)H(v^Tx)]|\\
&\leq |\EE_{x}[H_{\beta}^{+}(u^Tx)(H_{\beta}^{+}(v^Tx)-H(v^Tx))]+ \EE_{x}[(H_{\beta}^{+}(u^Tx)-H(u^Tx))H(v^Tx)]|\\
&\stackrel{(i)}{\leq}\EE_{x}|H_{\beta}^{+}(u^Tx)-H(u^Tx)|+ \EE_{x}|H_{\beta}^{+}(v^Tx)-H(v^Tx)|  \\ 
&\stackrel{(ii)}{=} 2\int_{0}^\beta (1-\frac{t}{\beta})p_{d}(t)\dd t \stackrel{(iii)}{\lesssim} \beta,
\end{align*}
where $(i)$ follows from  the boundedness of $H,H_{\beta}^{+}$, $(ii)$  follows from the the fact that the input distribution is $\rho=\text{Unif}(\sqrt{d}\SS^{d-1})$ and $p_d(\cdot)$ denotes the distribution of $X_1$ for $X\sim \unif(\sqrt{d}\SS^{d-1})$. 

Similarly, we can obtain 
\begin{equation}\label{eqn: 76}
\sup_{u,v\in\SS^{d-1}}|\varphi_{2,\beta}^{-}(u^Tv) - \varphi_{2}(u^Tv)|\lesssim \beta.
\end{equation}

\item \textbf{Step 2: concentration through smoothing.} Let $h_\beta^{-}(x;\theta)=H_\beta^{-}(x^Tu)H_\beta^{-}(x^Tv)$. Note that for any $x\in \sqrt{d}\SS^{d-1}$, we have
\begin{align*}
|h_\beta^{-}(x;\theta_1) - h_\beta^{-}(x;\theta')| &\leq |H_\beta^{-}(x^Tu)H_\beta^{-}(x^Tv)-H_\beta^{-}(x^Tu)H_\beta^{-}(x^Tv')| + |H_\beta^{-}(x^Tu)H_\beta^{-}(x^Tv')- H_\beta^{-}(x^Tu')H_\beta^{-}(x^Tv')|\\ 
&\leq |H_\beta^{-}(x^Tv)-H_\beta^{-}(x^Tv')| + |H_\beta^{-}(x^Tu)- H_\beta^{-}(x^Tu')|\\ 
&\leq \frac{1}{\beta}(|x^Tv-x^Tv'| + |x^Tu-x^Tu'|)\leq\frac{\sqrt{d}}{\beta}(\|v-v'\|+\|u-u'\|).
\end{align*}
Hence, $h(x;\cdot)$ is $\frac{\sqrt{d}}{\beta}$-Lipschitz  in $(\Omega, \|\cdot\|_{\Omega})$. In addition, since $\sup_t|H_\beta^-(t)|\leq 1$, we have  $\|h(X;\theta)\|_{\psi_1}\lesssim 1$ for any $\theta\in\Omega$. Then, applying Lemma \ref{lemma: covering-concentration} gives that  \wp at least $1-\delta$ it holds that
\[
  \sup_{u,v\in\SS^{d-1}}|\fn\sumin H_\beta^{-}(u^Tx)H_\beta^{-}(v^Tx) - \varphi_{2,\beta}^{-}(u^Tv)|\lesssim \frac{\sqrt{d}\epsilon}{\beta} + \max\left\{\sqrt{\frac{d\log(1/(\epsilon \delta))}{n}},\frac{d\log(1/(\epsilon \delta))}{n}\right\}
\]
By $H(t)\geq H_\beta^-(t)$ and the above inequality, we have
\begin{align*}
\hph_2(u^Tv) &= \fn\sumin H(u^Tx_i)H(v^Tx_i)\geq \fn\sumin H_\beta^{-}(u^Tx_i)H_\beta^{-}(v^Tx_i)\\
&\geq \varphi_{2,\beta}^{-}(u^Tv) - C\left(\frac{\sqrt{d}\varepsilon }{\beta}+\max\left\{\sqrt{\frac{d\log(1/(\epsilon \delta))}{n}},\frac{d\log(1/(\epsilon \delta))}{n}\right\}\right)\\
&\geq \varphi_{2}(u^Tv) - C\left(\beta + \frac{\sqrt{d}\varepsilon }{\beta}+\max\left\{\sqrt{\frac{d\log(1/(\epsilon \delta))}{n}},\frac{d\log(1/(\epsilon \delta))}{n}\right\}\right),
\end{align*}
where the last step uses \eqref{eqn: 76}. Optimizing $\beta$ gives 
\[
  \hph_2(u^Tv)\geq \varphi_{2}(u^Tv) - C \left((\sqrt{d}\varepsilon)^{1/2}+ \max\left\{\sqrt{\frac{d\log(1/(\epsilon \delta))}{n}},\frac{d\log(1/(\epsilon \delta))}{n}\right\}\right).
\]
Taking $\varepsilon=1/n$ and applying $n\geq d$, the above inequality can be simplified as 
\begin{equation}\label{eqn: app-2}
  \hph_2(u^Tv)\geq \varphi_{2}(u^Tv) - C\max\left\{\sqrt{\frac{d\log(n/\delta)}{n}},\frac{d\log(n/\delta)}{n}\right\}.
\end{equation}

Similarly, by utilizing $H_{\beta}^{+}$ and $\varphi_{2,\beta}^{+}$, we can prove 
\begin{align}\label{eqn: app-3}
\hph_2(u^Tv)\leq \varphi_{2}(u^Tv) + C\max\left\{\sqrt{\frac{d\log(n/\delta)}{n}},\frac{d\log(n/\delta)}{n}\right\}.
\end{align}
Combining \eqref{eqn: app-2} and \eqref{eqn: app-3}, we complete the proof.

\end{itemize}

\end{proof}

% \paragraph*{Proof of Theorem \ref{thm: empirical-relu-landscape}.}

\subsection{Proof of Theorem \ref{thm: sharpness-relu-landscape}}
\label{sec: proof-2lnn-sharpness}

\paragraph*{The expression of Fisher matrix.}
Notice that 
$
  \nabla_{w_j} f(x;\theta) = a_j\sigma'(w_j^Tx)x,\nabla_{a_j} f(x;\theta) = \sigma(w_j^Tx). 
$
Then the Fisher matrix is given by 
\begin{equation}\label{eqn: hessian-2lnn}
    G(\theta) = \begin{pmatrix}
    F_{1,1} &F_{1,2} &\dots & F_{1,m}\\
    F_{2,1} & F_{2,2}& \dots& F_{2, m}\\
    \vdots &\vdots &\ddots & \vdots \\
    F_{m,1} & F_{m,2}& \hdots& F_{m,m} 
    \end{pmatrix}\in \RR^{m (d+1)\times m(d+1)},
\end{equation}
where for any $j,k\in [m]$ the submatrix $F_{j,k}\in\RR^{(d+1)\times (d+1)}$ is given by 
\begin{equation}\label{eqn: Hjk}
    F_{j,k} = \begin{pmatrix}
    \hEE[\sigma(w_j^Tx)\sigma(w_k^Tx)] & \hEE[\sigma(w_j^Tx)a_k\sigma'(w_k^Tx)x^T]\\
    \hEE[\sigma(w_k^Tx)a_j\sigma'(w_j^Tx)x] & a_j a_k \hEE[\sigma'(w_j^Tx)\sigma'(w_k^Tx)xx^T]
    \end{pmatrix}.
\end{equation}

\paragraph*{Proof of Theorem  \ref{thm: sharpness-relu-landscape}.} We consider the case of trace, Frobenius norm, and the spectral norm separately. Specifically, combining Proposition \ref{pro: 2lnn-trace}, \ref{pro: 2lnn-spectral-norm}, and \ref{pro: 2lnn-fro}, we complete the proof.  The proofs of these propositions are provided in the subsequent sections.

\subsubsection{The trace of Fisher matrix}

\begin{proposition}[The trace]\label{pro: 2lnn-trace}
Recall that $N(d,\delta):=\inf\{n: d\log(1/\delta)/n\leq 1\}$.
For any $\delta\in (0,1)$, let $n\geq N(d,\delta)$. Then, \wp $1-\delta$, we have, $\tr(G(\theta))\sim \sum_j (\|w_j\|^2 + da_j^2)$. 
\end{proposition}
\begin{proof}
It is easy to show that 
\begin{align}\label{eqn: trace-emp}
\tr(G(\theta)) &=\sum_{j=1}^m \big(\|w_j\|^2 \hphi_1(\hw_j,\hw_j)+d a_j^2 \hphi_2(\hw_j,\hw_j)\big).
\end{align}
By Lemma \ref{lemma: empirical-kernel}, we have  
$
 \hph_i(u,u) \sim \varphi_i(u,u) - r_n\sim 1- o(1),
$
 where the last inequality is due to Lemma \ref{lemma: kernel-property} and the condition $n\geq N(d,\delta)$. Plugging this  into \eqref{eqn: trace-emp}, we complete the proof.
\end{proof}

\subsubsection{The Frobenius norm of Fisher matrix}
To help the estimate of Frobenius norm, we define for $u,v\in\SS^{d-1}$ that
\[
    b_{u,v} = \hEE[\sigma(u^Tx)\sigma'(v^Tx)x]\in\RR^d,\quad A_{u,v}=\hEE[\sigma'(u^Tx)\sigma'(v^Tx)xx^T]\in\RR^{d\times d}.
\]
Then by \eqref{eqn: hessian-2lnn}, we have 
\begin{align}\label{eqn: 61}
\|G(\theta)\|_F^2 &= \sum_{j,k=1}^m \left(\|w_j\|^2\|w_k\|^2\hph_1(\hw_j,\hw_k)^2 +a_j^2a_k^2 \|A_{\hw_j,\hw_k}\|_F^2 + a_j^2\|w_k\|^2 \|b_{\hw_k,\hw_j}\|^2_2 + a_k^2\|w_j\|^2 \|b_{\hw_j,\hw_k}\|^2_2\right).
\end{align}
Next, we  bound each term of the right hand side separately.
 
\begin{lemma}\label{lemma: buv}
For any $\delta\in (0,1)$, if $n\gtrsim N(d,\delta)$, \wp $1-\delta$ it holds that
$
  \sup_{u,v\in\SS^{d-1}}\|b_{u,v}\|_2\lesssim 1.
$
\end{lemma}
\begin{proof}
Note that for any $u,v\in\SS^{d-1}$, 
\begin{align}\label{eqn: 121}
\notag \|b_{u,v}\| &= \sup_{\|w\|=1}w^Tb_{u,v} = \sup_{\|w\|=1}\hEE[\sigma(u^Tx)\sigma'(v^Tx)w^Tx]\\
\notag &\lesssim \sup_{\|w\|=1}\hEE[\sigma(u^Tx)w^Tx]\leq \sup_{\|w\|=1}\sqrt{\hEE[\sigma^2(u^Tx)]}\sqrt{\hEE[|w^Tx|^2]} \qquad\qquad (\text{ Cauchy-Schwartz})\\ 
\notag &=\sup_{w\in\SS^{d-1}} \sqrt{\hph_1(u,u)}\sqrt{w^T\hSigma_nw} = \sqrt{\hph_1(u,u)}\lambda_{\max}(\hSigma_n)\lesssim 1,
\end{align}
where the last steps follows from Lemma \ref{lemma: empirical-kernel} and Lemma \ref{lemma: covariance-gaussian}.
\end{proof}

\begin{lemma}\label{lemma: Auv}
For any $\delta\in (0,1)$, if $n\gtrsim d + \log(1/\delta)$, then \wp $1-\delta$ it holds for any $u,v\in\SS^{d-1}$ that 
$$
\sqrt{d}\hphi_2(u,v)\leq \|A_{u,v}\|_F\lesssim \sqrt{d}.
$$
\end{lemma}
\begin{proof}
\underline{\textbf{Upper bound.}} We first prove a more general result. Let $a\in\RR^n$ with $\sup_{i\in [n]}|a_i|\lesssim 1$ and $Q_a = \fn\sumin a_i x_ix_i^T$. 
Then, 
\[
\|\fn\sumin a_i x_ix_i^T\|_F^2 = \frac{1}{n^2}\sum_{i,j=1}^n a_i a_j (x_i^Tx_j)^2\lesssim \frac{1}{n^2}\sum_{i,j=1}^n  (x_i^Tx_j)^2 = \|\fn\sumin x_ix_i^T\|_F^2 = \|\hSigma_n\|_F^2.
\]
By Lemma \ref{lemma: covariance-gaussian}, \wp $1-\delta$ that $\|\hSigma_n\|_F\lesssim \sqrt{d}$. Thus, $\|Q_a\|_F\leq \sqrt{d}$ for any $\|a\|_\infty \lesssim 1$. 
Notice that we can rewrite $A_{u,v}$ as
$
  A_{u,v} = \fn\sumin \sigma'(u^Tx_i)\sigma'(v^Tx_i)x_ix_i^T,
$
with $|\sigma'(u^Tx_i)\sigma'(v^Tx_i)|\lesssim 1$.  Thus, $\|A_{u,v}\|_F\leq \sqrt{d}$.

\underline{\textbf{Lower bound.}} Now we consider the lower bound:
\begin{align*}
\|A_{u,v}\|_F &\geq \frac{1}{\sqrt{d}}\trace(A_{u,v}) = \frac{1}{\sqrt{d}} \sum_{j=1}^d \hat{\EE}[\sigma'(u^Tx)\sigma'(v^Tx)x_j^2] \\ 
&= \sqrt{d} \hat{\EE}[\sigma'(u^Tx)\sigma'(v^Tx)] = \sqrt{d}\, \hph_2(u,v).
\end{align*}
\end{proof}

\begin{lemma}\label{lemma: kernel-spd}
For any $\delta\in (0,1)$, if $n\gtrsim N(d,\delta)$, then \wp at least $1-\delta$ it holds for $i=1,2$ that 
\[
   \sum_{j,k=1}^m \alpha_j^2\alpha_k^2 \hph_i(\hw_j,\hw_k)\sim \sum_j \alpha_j^2.
\]
\end{lemma}
\begin{proof}
WLOG, assume $\sum_j \alpha_j^2=1$ and let $\Delta_{j,k}=\varphi_1(\hw_j,\hw_k)-\hph_1(\hw_j,\hw_k)$. 
By Lemma \ref{lemma: empirical-kernel}, \wp at least $1-\delta$ it holds that 
$\sup_{j,k\in [m]}|\Delta_{j,k}|\leq r_n$, where $r_n$ is defined in Lemma \ref{lemma: empirical-kernel}. Hence,
\begin{align}\label{eqn: 63}
\notag \sum_{j,k=1}^m \alpha_j^2\alpha_k^2 \hph_1(\hw_j,\hw_k)^2 
&= \sum_{j,k=1}^m \alpha_j^2\alpha_k^2 (\varphi_1(\hw_j,\hw_k)+\Delta_{j,k})^2\\
\notag &= \sum_{j,k=1}^m \alpha_j^2\alpha_k^2 \varphi_1(\hw_j,\hw_k)^2+2\sum_{j,k=1}^m \alpha_j^2\alpha_k^2 \varphi_1(\hw_j,\hw_k)\Delta_{j,k} +\sum_{j,k=1}^m \alpha_j^2\alpha_k^2\Delta_{j,k}^2\\
&\gtrsim (1+O(r_n^2))\sum_{j,k=1}^m \alpha_j^2\alpha_k^2 +O(r_n) \sum_{j,k=1}^m \alpha_j^2\alpha_k^2,
\end{align}
where the last step follows the third conclusion in Lemma \ref{lemma: kernel-property} and the fact that $\sup_t|\varphi_1(t)|\lesssim 1$.

Taking $n$ to be large enough, we complete the proof. The case of $i=2$ follows the same proof procedure.
\end{proof}

Now we are ready to prove the main proposition.
\begin{proposition}[The Frobenius norm]\label{pro: 2lnn-fro}
For any $\delta\in (0,1)$, let $n\geq N(d,\delta)$. Then, \wp $1-\delta$, we have $\|G(\theta)\|_F\sim \sum_j (\|w_j\|^2 + \sqrt{d}a_j^2)$. 
\end{proposition}
\begin{proof}
\underline{\textbf{Lower bound.}} By \eqref{eqn: 61}, \wp at least $1-\delta$ that 
\begin{align}\label{eqn: 62}
\notag \|G(\theta)\|_F^2 &\geq \sum_{j,k=1}^m \left(\|w_j\|^2\|w_k\|^2\hph_1(\hw_j,\hw_k)^2  + a_j^2 a_k^2  A_{\hw_j,\hw_k}^2\right)\\ 
\notag &\geq \sum_{j,k=1}^m \|w_j\|^2\|w_k\|^2 \hph_1(\hw_j,\hw_k)^2 + d\sum_{j,k=1}^m a_j^2a_k^2 \hph_2(\hw_j,\hw_k)^2\qquad \text{ (Use Lemma \ref{lemma: Auv})}\\ 
\notag &\gtrsim (\sum_{j}^m \|w_j\|^2)^2 + (\sqrt{d}\sum_j a_j^2)^2 \qquad \text{ (Use Lemma \ref{lemma: kernel-spd})}\\ 
&\geq \frac{1}{2}\left(\sum_{j} \|w_j\|^2 + \sqrt{d}\sum_j a_j^2\right)^2 \qquad \text{( $(x^2+y^2)\geq (x+y)^2/2$)}.
\end{align}

\underline{\textbf{Upper bound.}}
By Lemma \ref{lemma: buv} and \ref{lemma: Auv}, we have \wp $1-\delta$  that 
$
\|b_{\hw_j,\hw_k}\|\lesssim 1, \|A_{\hw_j,\hw_k}\|_F\lesssim \sqrt{d}.
$
Substituting it into \eqref{eqn: 61} gives 
\begin{align}
\notag \|G(\theta)\|_F &\leq \sum_{j,k=1}^m (\|w_j\|^2\|w_k\|^2 \hph_1(\hw_j,\hw_k)^2 + \sqrt{d}a_j^2\|w_k\|^2 + \sqrt{d}a_k^2\|w_j\|^2 + da_j^2a_k^2)\\ 
\notag &\lesssim \sum_{j,k=1}^m (\|w_j\|^2\|w_k\|^2 + \sqrt{d}a_j^2 + \sqrt{d}a_k^2 + da_j^2a_k^2)\qquad \,\text{ (Use Lemma \ref{lemma: kernel-spd})}\\ 
\notag & =\sum_{j,k} (\|w_j\|^2+\sqrt{d}a_j^2) (\|w_k\|^2+\sqrt{d}a_k^2) = (\sum_j (\|w_j\|^2+\sqrt{d}a_j^2))^2.
\end{align}
\end{proof}

\subsubsection{The spectral norm}

To control the spectral norm, we need again to handle the discontinuity of $\sigma'$ at the origin. 
Define 
\begin{equation}\label{eqn: 66}
\begin{aligned}
\phi_{\beta}^{+}(u^Tv) = \EE_x[H_{\beta}^{+}(u^Tx)\sigma(v^Tx)]&,\quad \phi^{-}_\beta(u^Tv) = \EE_x[H_\beta^{-}(u^Tx)\sigma(v^Tx)]\\ 
\phi(u^Tv) &= \EE[H(u^Tx)\sigma(v^Tx)].
\end{aligned}
\end{equation}

\begin{lemma}\label{lemma: phi-2lnn}
$|\phi_\beta^{+}(u^Tv)-\phi(u^Tv)|\lesssim \beta$ and $|\phi_\beta^{-}(u^Tv)-\phi(u^Tv)|\lesssim \beta$.
\end{lemma}

\begin{proof}
Note that 
\begin{align}\label{eqn: 71}
\notag |\phi_\beta^{-}(u^Tv)-\phi(u^Tv)| &= |\EE[H_{\beta}^-(u^Tx)\sigma(v^Tx)] - \EE[H(u^Tx)\sigma(v^Tx)]| \\ 
\notag &=\EE[(H_\beta^{-}(u^Tx)-H(u^Tx))\sigma(v^Tx)] \leq \sqrt{\EE[\sigma^2(v^Tx)]}\sqrt{\EE[(H_\beta^{-}(u^Tx)-H(u^Tx))^2]}\\ 
\notag &\lesssim \sqrt{\EE_{z}[|H_\beta^{-}(z)-H(z)|^2]} = \sqrt{\int_0^\beta(1-\frac{z}{\beta})^2 p_d(z)\dd z}\qquad \text{(Let $z=u^Tx$)}\\ 
&\lesssim \sqrt{\int_0^\beta (1-\frac{z}{\beta})^2\dd z}\lesssim \beta,
\end{align}
where $p_d$ is the density function of $u^TX$ for $X\sim\mathrm{Unif}(\sqrt{d}\SS^{d-1})$ and we use the fact that $\sup_z |p_d(z)|\lesssim 1$.
Similarly, we can prove the case of $\phi_\beta^{+}$.
\end{proof}

\begin{lemma}\label{lemma: yy}
For any $\delta\in (0,1)$, if $n\gtrsim N(d,\delta)$, then \wp at least $1-\delta$ that 
\[
   \sup_{u,v\in\SS^{d-1}}| \hEE[\sigma'(u^Tx)v^Tx] - \EE[\sigma'(u^Tx)v^Tx]| \lesssim r_n.
\]
\end{lemma}
\begin{proof}
The proof is essentially similar to the proof of Lemma \ref{lemma: empirical-kernel}. 
\begin{itemize}
\item \textbf{Step 1.} 
Note that $\|H_\beta^{+}(X)\|_{\psi_2}\lesssim 1$ and $\|\sigma(v^TX)\|_{\psi_2}\lesssim 1$. By Lemma \ref{lemma: pro-orlic-norm}, we have 
\[
\|H_\beta^{+}(X)\sigma(v^TX)\|_{\psi_1}\leq  \|H_\beta^{+}(X)\|_{\psi_2}\|\sigma(v^TX)\|_{\psi_2}\lesssim 1.
\]

The fact that $H_\beta$ is $\frac{1}{\beta}$-Lipschitz continuous and $\sigma$ is $1$-Lipschitz implies that for any $x\in \sqrt{d}\SS^{d-1}$, $J_x(u,v):=H_\beta^{+}(u^Tx)\sigma(v^Tx)$ is $\frac{d}{\beta}$ Lipschitz with respect to the metric $\|(u,v)-(u',v')\|_{\Omega}:=\|u-u'\|+\|v-v'\|$. In addition, by Lemma \ref{lemma: 2sphere}, the covering number of $\Omega=\SS^{d-1}\otimes\SS^{d-1}$ with respect to this metric is $N_\epsilon = (6/\epsilon)^{2d}$.
Then, by Lemma \ref{lemma: covering-concentration}, we have 
\begin{equation}\label{eqn: 68}
    \sup_{u,v\in\SS^{d-1}} |\hEE_x[H_{\beta}^{+}(u^Tx)\sigma(v^Tx)]-\phi_\beta^{+}(u^Tv)| \lesssim \frac{d}{\beta}\epsilon + \max\{\sqrt{\frac{d\log(1/(\epsilon\delta))}{n}},\frac{d\log(1/(\epsilon\delta))}{n}\}.
\end{equation}
Similarly, we can obtain the following holds \wp at least $1-\delta$,
\begin{equation}\label{eqn: 69}
    \sup_{u,v\in\SS^{d-1}} |\hEE_x[H_{\beta}^{-}(u^Tx)\sigma(v^Tx)]-\phi_\beta^{-}(u^Tv)| \lesssim \frac{d}{\beta}\epsilon + \max\{\sqrt{\frac{d\log(1/(\epsilon\delta))}{n}},\frac{d\log(1/(\epsilon\delta))}{n}\}.
\end{equation}
\item \textbf{Step 2.} Noting $t=\sigma(t)-\sigma(-t)$ for any $t\in\RR$, we have
\begin{align}\label{eqn: 70}
\notag \hEE[\sigma'(u^Tx)v^Tx] &= \hEE[H(u^Tx)\sigma(v^Tx)] - \hEE[H(u^Tx)\sigma(-v^Tx)]\\ 
\notag &\geq \hEE[H_{\beta}^-(u^Tx)\sigma(v^Tx)] - \hEE[H_\beta^{+}(u^Tx)\sigma(-v^Tx)]\\ 
&\geq \phi_\beta^{-}(u^Tv) -\phi_\beta^{+}(-u^Tv) - C\left(\frac{d}{\beta}\epsilon + \max\{\sqrt{\frac{d\log(1/(\epsilon\delta))}{n}},\frac{d\log(1/(\epsilon\delta))}{n}\}\right),
\end{align}
where the last inequality follows from \eqref{eqn: 68} and \eqref{eqn: 69}.

\item \textbf{Step 3.} Applying Lemma \ref{lemma: phi-2lnn} to \eqref{eqn: 70} gives 
\begin{align*}
\hEE[\sigma'(u^Tx)v^Tx] &\geq \phi(u^Tv) -\phi(-u^Tv) - C\left(\beta + \frac{d}{\beta}\epsilon + \max\{\sqrt{\frac{d\log(1/(\epsilon\delta))}{n}},\frac{d\log(1/(\epsilon\delta))}{n}\}\right)\\ 
&= \EE[\sigma'(u^Tx)v^Tx] - C\left(\beta + \frac{d}{\beta}\epsilon + \max\{\sqrt{\frac{d\log(1/(\epsilon\delta))}{n}},\frac{d\log(1/(\epsilon\delta))}{n}\}\right).
\end{align*}
Optimizing $\beta$ and taking $\epsilon=1/n$, we obtain 
\begin{align*}
\hEE[\sigma'(u^Tx)v^Tx]&\geq \EE[\sigma'(u^Tx)v^Tx] - C\left(d^{1/2}\epsilon^{1/2} + \max\{\sqrt{\frac{d\log(1/(\epsilon\delta))}{n}},\frac{d\log(1/(\epsilon\delta))}{n}\}\right) \\ 
&\geq \EE[\sigma'(u^Tx)v^Tx] - C  r_n.
\end{align*}

Analogously, we can prove the other side inequality.
\end{itemize}
\end{proof}

Let $x\in\RR^d$, we use $x_{,k}$ to denote the $k$-th coordinate of $x$, which is distinguished from $x_k$, denoting the $k$-th sample in the training set. In addition, we use $\{e_k\}_{k=1}^d$ to denote the standard basis of $\RR^d$.
\begin{lemma}\label{lemma: zz}
$\sum_{k=1}^d(\EE[\sigma'(u^Tx)x_{,k}])^2\gtrsim 1$
\end{lemma}
\begin{proof}
Note that $x_{,k}=g(e_k^Tx)$ with $g(x)=x$. Let $g(x)=\sum_s G_s h_s(x)$ be the Hermite expansion of $g$. By \eqref{eqn: hermite-function}, we have $G_k=1$ for $k=1$ and $0$ otherwise. Then Lemma \ref{lemma: app-hermite} gives that for any $u\in\SS^{d-1}$,
\[
    \EE[\sigma'(u^Tx)x_{,k}] = \EE[\sigma'(u^Tx)g(e_k^Tx)] = \sum_{s} \beta_s G_s (u^Te_k)^s = \beta_1 u_{k}. 
\]
Hence, 
$% \begin{align*}
\sum_{k=1}^d (\EE[\sigma'(u^Tx)x_{,k}])^2 = \sum_{k=1}^d (\beta_1 u_k)^2 = \beta_1^2 \|u\|^2=\beta_1^2\gtrsim 1.
$
\end{proof}

\begin{proposition}[The spectral norm]\label{pro: 2lnn-spectral-norm}
For any $\delta\in (0,1)$, if $n\gtrsim d N(d,\delta)$, \wp $1-\delta$ we have
$\|G_\theta\|_2\sim \sumjm (w_j^2 + a_j^2)$.
\end{proposition}
\begin{proof}
Let $\Phi=(\nabla f(x_1;\theta),\nabla f(x_2;\theta),\dots,\nabla f(x_n;\theta))\in\RR^{p\times n}$. Then, $G(\theta) = \Phi\Phi^T\in\RR^{p\times p}$. Then
\begin{align}\label{eqn: relu-l2-norm}
\notag \|G(\theta)\|_2 &= \lambda_{\max}(\Phi\Phi^T)  = \lambda_{\max}(\Phi^T\Phi)=\sup_{u\in\SS^{n-1}}u^T\Phi^T\Phi u = \sup_{u\in\SS^{n-1}}\|\Phi u\|^2\\ 
&=\sup_{\|u\|_{L^2(\hat{\rho})}=1} \sumjm \left((\hEE[\sigma(w_j^Tx)u(x)])^2 + a_j^2 \sum_{k=1}^d(\hEE[\sigma'(w_j^Tx)x_{,k}u(x)])^2\right),
\end{align}
where $\hat{\rho}=\fn\sumin \delta(x_i-\cdot)$. 

\underline{\textbf{Lower bound.}}
Taking $u(x)\equiv 1$, we obtain
\begin{align}\label{eqn: app-190}
\notag \|G(\theta\|_2 &\geq \sumjm \left((\hEE[\sigma(w_j^Tx)])^2 + a_j^2 \sum_{k=1}^d(\hEE[\sigma'(w_j^Tx)x_{,k}])^2\right)\\
&= \sumjm \left(\|w_j\|^2 \hph_1(\hw_j,\hw_j) + a_j^2 \sum_{k=1}^d(\EE[\sigma'(w_j^Tx)x_{,k}])^2\right).
\end{align}
Note that
\begin{itemize}
\item By Lemma \ref{lemma: empirical-kernel}, \wp at least $1-\delta$ that $\hph_1(u,u)\gtrsim 1$ for any $u\in\SS^{d-1}$.
\item In addition,
\begin{align*}
    \sum_{k=1}^d(\hEE[\sigma'(w_j^Tx)e_k^Tx])^2 &= \sum_{k=1}^d(\EE[\sigma'(w_j^Tx)e_k^Tx] + O(r_n))^2\qquad \text{(Lemma \ref{lemma: yy})}\\ 
    &\geq \sum_{k=1}^d(\EE[\sigma'(w_j^Tx)e_k^Tx])^2 + O(r_n)\sum_{k=1}^d\EE[\sigma'(w_j^Tx)e_k^Tx]\\ 
    &\gtrsim 1-O(\sqrt{d}r_n),
\end{align*}
where the last step uses Lemma \ref{lemma: zz}  and 
$
    |\sum_{k=1}^d \EE[\sigma'(u^Tx)e_k^Tx]| = |\EE[\sigma'(u^Tx)(\sum_k e_k)x]|\leq \sqrt{o^T\EE[xx^T]o}\leq \|o\|=\sqrt{d},
$
where $o=(1,1,\dots,1)\in\RR^d$.
\end{itemize}
 Plugging the above two estimates  into \eqref{eqn: app-190}, we obtain
$
    \|G_\theta\|\gtrsim \sumjm (w_j^2 + a_j^2). 
$

\underline{\textbf{Upper bound.}}
In addition, \eqref{eqn: relu-l2-norm} also implies
\begin{align}\label{eqn: 74}
\notag \|G_\theta\| &\leq \sumjm \left(\sup_{\|u\|_{L^2(\hat{\rho})}=1}(\hEE[\sigma(w_j^Tx)e(x)])^2  + a_j^2 \sup_{\|u\|_{L^2(\hat{\rho})}=1} \sum_{k=1}^d(\hEE[\sigma'(w_j^Tx)x_{,k}u(x)])^2\right)\\
&\leq \sumjm \left(\hEE[\sigma(w_j^Tx)^2] + a_j^2 \|\hEE_{x}[\sigma'(w^Tx)^2xx^T]\|^2_2\right)
\end{align}
where the last step follows from the Cauchy-Schwarz inequality and Lemma \ref{lemma: variation-principle}. Then, the proof of the upper bound is completed by plugging the following estimates into \eqref{eqn: 74}.
\begin{itemize}
\item By Lemma \ref{lemma: empirical-kernel}, \wp $1-\delta$ that
$\hEE[\sigma(w_j^Tx)^2]=\|w_j\|^2\hph_1(\hw_j,\hw_j)\lesssim \|w_j\|^2$.
\item In addition,
\begin{align*}
\|\EE_{x}[\sigma'(w_j^Tx)^2xx^T]\|_2 &= \sup_{v\in\SS^{d-1}} v^T\EE_x[\sigma'(w_j^Tx)^2xx^T]v\\
&=\EE_{x}[\sigma'(w_j^Tx)^2(v^T x)^2]\lesssim \EE_{x}[|v^Tx|^2]\lesssim\sup_{v\in\SS^{d-1}}v^T\EE[xx^T]v = 1.
\end{align*}
\end{itemize}

Lastly, combining the lower and upper bound, we complete the proof.
\end{proof}

\subsection{Proof of Theorem \ref{thm: 2lnn-generalization-bound}}
\label{sec: proof-2lnn-gengap}

Our proof needs the following path-norm based generalization bound:
\begin{proposition}\label{pro: 2lnn-sharpness-genbound}
Suppose $\sup_{x\in \cX}|f^*(x)|\leq 1$ and $\gamma\geq 1$. If $\htheta$ is a global minimum of $\erisk(\cdot)$ satisfying $\|\htheta\|_{\cP}\leq \gamma$, then 
$
  \risk(\htheta)\leq (\log^3(n)+\log(1/\delta))\frac{d\gamma^2}{n}.
$
\end{proposition}

\begin{proof}
Let $\cF_{\gamma}=\{f(\cdot;\theta):\|\theta\|_{\cP}\leq \gamma\}$. Then, it is easy to show that the worst-case Rademacher complexity (see Eq.~\eqref{eqn: worst-case-rad}) $\mathfrak{R}_n(\cF_\gamma)\lesssim \sqrt{d}\gamma/\sqrt{n}$. In addition,
\[
    |f(x)|=|\sum_j a_j\sigma(w_j^Tx)|\leq \sum_j |a_j||w_j^Tx|\leq \|\theta\|_{\cP}\|x\|\leq \gamma\sqrt{d}.
\]
Hence, the loss function $\phi(t)=t^2/2$ satisfying $|\phi''|\lesssim 1$ and $|\phi|\lesssim \gamma^2d$. Then, by Theorem \ref{thm: fast-rates-smoothloss}, we have 
$$
  \risk(\htheta)\lesssim \frac{\log^3(n)d\gamma^2}{n} + \frac{d\gamma^2\log(1/\delta)}{n}.
$$
\end{proof}

\paragraph*{Proof of Theorem \ref{thm: 2lnn-generalization-bound}.}

For $\htheta_{\mathrm{sgd}}$, by Proposition \ref{pro: stability-1} and Theorem \ref{thm: sharpness-relu-landscape} we have \wp at least $1-\delta$ that 
\begin{align}
\frac{2}{\eta} &\geq \tr(G(\ttheta_{\mathrm{sgd}}))\sim \|\ttheta_{\mathrm{sgd}}\|_{1,d} \geq \sqrt{d}\|\ttheta_{\mathrm{sgd}}\|_{\cP}.
\end{align}
Hence, $\|\htheta_{\mathrm{sgd}}\|_{\cP}\lesssim 2/(\eta\sqrt{d})$. Applying Proposition \ref{pro: 2lnn-sharpness-genbound}, we obtain 
\[
    \risk(\htheta_{\mathrm{sgd}})\lesssim \frac{\log^3(n)+\log(1/\delta)}{n\eta^2}.
\]

Similarly, we have $\|\htheta_{\mathrm{gd}}\|_{\cP}\lesssim 2/\eta$. Applying Proposition \ref{pro: 2lnn-sharpness-genbound} completes the proof.
\qed

\section{Missing Proofs in Section \ref{sec: diagonal nets}}
\label{sec: proof-diag-net}
%  

% \begin{lemma}\label{lemma: app-diag}
% Suppos $x_i\stackrel{iid}{\sim}\cN(0,S)$ for $i=1,\dots,n$. Let $x_i=(x_{i,1},x_{i,2},\dots,x_{i,d})$ for $i=1,\dots,n$.  Then for any $\delta\in (0,1)$, \wp $1-\delta$ over the sampling of $x_1,\dots,x_n$ we have
% \[
%   \sup_{j,k}|\fn\sumin x_{i,j}x_{i,k}-S_{j,k}|\lesssim \max(\frac{d\log(1/\delta)}{n}).
% \]
% \end{lemma}
% \begin{proof}
% Notice that $\EE[x_{i,j}x_{i,k}]=S_{j,k}$ and $\|x_{i,j}x_{i,k}\|_{\psi_1}\lesssim 1$. Then by Bernstein's inequality, we have
% \[
%   \PP\left\{|\fn\sumin x_{i,j}x_{i,k}-S_{j,k}|\geq t\right\}\leq 2e^{-C n \min(t^2,t)}.
% \]
% Taking the union bound gives 
% \[
%   \PP\left\{\sup_{j,k\in [d]}|\fn\sumin x_{i,j}x_{i,k}-S_{j,k}|\geq t\right\}\leq 2d^2e^{-C n \min(t^2,t)}.
% \]
% Let the failure probability to be smaller than $\delta$. Then, we obtain \wp $1-\delta$ that for any $j,k\in [d]$,
% \[
%   |\fn\sumin x_{i,j}x_{i,k}-S_{j,k}|\lesssim \max\left(\frac{\log(d/\delta)}{n},\sqrt{\frac{\log(d/\delta)}{n}}\right)=:\varepsilon_n.
% \]
% \end{proof}
We first need the following lemma.
\begin{lemma}\label{lemma: covariance-point-wise}
Let $X$ be a mean-zero with $\|X\|_{\psi_2}\lesssim 1$. Let $\Sigma=\EE[XX^T]$ and $\hSigma_n =\fn\sumin X_iX_i^T$ be the population and empirical covariance matrix, respectively. For any $\delta\in (0,1)$,  if $n\gtrsim \log(d/\delta)$, we have 
\wp $1-\delta$ the following holds for any $j,k\in [d]$.
\[
|(\hSigma_n)_{j,k} - (\Sigma)_{j,k}|\lesssim \sqrt{\frac{\log(d/\delta)}{n}}.
\]
\end{lemma}
\begin{proof}
First, for each $j,k\in [d]$, $(\hSigma_n)_{j,k}=\fn\sumin X_{i,j}X_{i,k}$. By the sub-Gaussian property, $\|X_{i,j}X_{i,k}\|_{\psi_1}\lesssim \|X_{i,j}\|_{\psi_2}\|X_{i,k}\|_{\psi_2}\lesssim 1$. Thus, by Bernstein's inequality, we have
\[
  \PP\left\{|(\hSigma_n)_{j,k} - (\Sigma)_{j,k}|\geq t\right\}\lesssim e^{-C n\min(t,t^2)}.
\]
Then, taking the union bound, we obtain 
\[
\PP\left\{\sup_{j,k}|(\hSigma_n)_{j,k} - \Sigma_{j,k}|\geq t\right\}\lesssim d^2 e^{-C n\min(t,t^2)}.
\]
Hence, for any $\delta\in(0,1/e)$, \wp $1-\delta$ the following holds for any $j,k\in [d]$
\[
|(\hSigma_n)_{j,k} - (\Sigma)_{j,k}|\lesssim \sqrt{\frac{\log(d/\delta)}{n}}.
\]
\end{proof}

\subsection{Proof of Theorem \ref{thm: diagnet-sharpness}}
\label{sec: proof-diagnet-sharpness}

In this section, we prove Theorem \ref{thm: diagnet-sharpness} for the case of the spectral norm, Frobenius norm, and the trace separately.

\begin{lemma}[The spectral norm]
Suppose Assumption \ref{assumption: input-1} holds.
For any $\delta\in (0,1)$, let $n\gtrsim d+\log(1/\delta)$ such that $\varepsilon_n = \sqrt{\frac{d+\log(1/\delta)}{n}}\leq 1$. Then, \wp $1-\delta$, we have 
\[
(1-\varepsilon_n)\|\theta\|_\infty\leq  \|G(\theta)\|_2 \leq (1+\varepsilon_n) \|\theta\|_\infty.
\]
\end{lemma}
\begin{proof}
Let $z=(u,v)\in\RR^{2d}$ with $u,v\in\RR^d$ such that $\|z\|^2=\|u\|^2+\|v\|^2=1$. Then, we have 
\begin{align*}
\|G(\theta)\|_2 &= \sup_{\|z\|=1} z^TG(\theta) z = \sup_{\|z\|=1}\sum_{j,k=1}^n \left(u_i u_j a_i a_j \hs_{i,j} + u_i v_j a_i b_j \hs_{i,j}+v_i u_j b_i a_j \hs_{i,j} + v_i  v_j b_ib_j \hs_{i,j}\right)\\
&=\sup_{\|z\|=1}\left((u\circ a)^T\hSigma_n (u\circ a) + 2 (u\circ a)^T\hSigma_n (v\circ b) + (v\circ b)^T\hSigma_n (v\circ b)\right)\\
&= \sup_{\|z\|=1} (u\circ a+ v\circ b)^T\hSigma_n (u\circ a+ v\circ b)\\
&= \sup_{\|z\|=1} \left(\|u\circ a+ v\circ b\|^2 + (u\circ a+ v\circ b)^T(\hSigma_n-I) (u\circ a+ v\circ b)\right).
\end{align*}
By Lemma \ref{lemma: covariance-gaussian}, we have \wp at least $1-\delta$ that $\|\hSigma_n - I_d\|\leq \varepsilon_n$. Therefore,
\[
    (1-\varepsilon_n)\sup_{\|z\|=1} \|u\circ a+ v\circ b\|^2\leq \|G(\theta)\|_2\leq (1+\varepsilon_n) \sup_{\|z\|=1} \|u\circ a+ v\circ b\|^2.
\]
Noticing that 
\begin{align*}
\sup_{\|z\|=1} \|u\circ a+ v\circ b\|^2 &= \sup_{\|z\|=1} \sum_{j=1}^d (a_ju_j+b_j v_j)^2\\
&=\sup_{\|z\|=1} \sum_j (a_j^2+b_j^2)(u_j^2+v_j^2)\\
&= \sup_{t\in\SS^{d-1}}\sum_j (a_j^2+b_j^2)t_j^2\\
&=\max_{j} (a_j^2+b_j^2),
\end{align*}
we complete the proof.
\end{proof}

\begin{lemma}
For any $\delta\in (0,1)$, if $n\gtrsim \log(d/\delta)$, then \wp at least $1-\delta$ that $$
(1-\varepsilon_n)(\|a\|^2+\|b\|^2)\leq \tr(G(\theta))\leq (1+\varepsilon_n)(\|a\|^2+\|b\|^2).
$$
\end{lemma}
\begin{proof}
Notice that 
$
\tr(G(\theta))=\sum_{j=1}^d (a_i^2 \hs_{i,i} + b_i^2 \hs_{i,i}).
$
By Lemma \ref{lemma: covariance-point-wise}, we have \wp $1-\delta$, for any $i\in [n]$, $|\hs_{i,i}-1|\leq \varepsilon_n$. Combining them completes the proof.
\end{proof}

\begin{lemma}
For any $\delta\in (0,1)$, if $n\gtrsim \log(d/\delta)$, then \wp at least $1-\delta$ we have 
\[
   (1-\varepsilon_n)\|\alpha\|_2\leq  \|G(\theta)\|_F \leq \varepsilon_n \|\alpha\|_1+\sqrt{1+4\varepsilon_n}\|\alpha\|_2
\]
\end{lemma}
\begin{proof}
By the definition,
\begin{align}\label{eqn: 59}
\|G(\theta)\|_F^2=\sum_{j,k=1}^d \left(a_i^2 a_j^2 \hs^2_{i,j} + a_i^2 b_j^2 \hs_{i,j}+b_i^2 a_j^2 \hs_{i,j}^2 + b_i^2b_j^2 \hs_{i,j}^2\right)
\end{align}
By Lemma \ref{lemma: covariance-point-wise}, \wp at least $1-\delta$ we have $|s_{i,j}-\hs_{i,j}|\leq \varepsilon_n$ for any $i,j\in [n]$. Thus, 
\[
\hs_{i,j}^2=(\hs_{i,j}-s_{i,j}+s_{i,j})^2 \in  \begin{cases}
[0,\varepsilon_n^2] & \text{if } i\neq j \\
[(1-\varepsilon_n)^2, (1+\varepsilon_n)^2] & \text{if } i=j
\end{cases}.
\]
Plugging it into \eqref{eqn: 59} gives: 
\begin{align}\label{eqn: 131}
\notag \|G(\theta)\|_F^2&\leq \sum_{i\neq j}(a_i^2a_j^2+a_i^2b_j^2+b_i^2a_j^2 +b_i^2b_j^2)\varepsilon_n^2 + (1+\varepsilon_n)^2\sum_{i}(a_i^4+2a_i^2b_i^2+b_i^4)\\
\notag &\leq \sum_{i,j}(a_i^2a_j^2+a_i^2b_j^2+b_i^2a_j^2 +b_i^2b_j^2)\varepsilon_n^2 + ((1+\varepsilon_n)^2+\varepsilon_n^2)\sum_{i}(a_i^4+2a_i^2b_i^2+b_i^4)\\
&\leq \varepsilon_n^2 (\sum_i a_i^2+b_i^2)^2 + (1+2\varepsilon_n)^2 \sum_i (a_i^2+b_i^2)^2,
\end{align}
and 
\begin{align}\label{eqn: 132}
\notag \|G(\theta)\|_F^2&\geq \sum_{i}s_{i,i}^2 (a_i^4+2a_i^2b_i^2+b_i^4)\\
&\geq (1-\varepsilon_n)^2\sum_i (a_i^2+b_i^2)^2=(1-\varepsilon_n)^2\|\alpha\|^2_2.
\end{align}
Combining \eqref{eqn: 131} and \eqref{eqn: 132} completes the proof.
\end{proof}

\subsection{Proof of Theorem \ref{thm: diagnet-gen-gap}.} 
\label{sec: proof-diagnet-gen-gap}

For any $Q>0$, denote the class of linear predictors with bounded $\ell_1$ by  
$
\cH_Q = \{h_\beta: x\to \beta^Tx | \|\beta\|_1\leq Q\},
$
for which \citet[Lemma 26.11]{shalev2014understanding}  gives 
\[
   \mathfrak{R}_n(\cH_Q)= \sup_{x_1,\dots,x_n}\erad(\cH)\leq \max_{i\in [n]}\|x_i\|_\infty Q\sqrt{\frac{2\log(2d)}{n}}\leq Q\sqrt{\frac{2\log(2d)}{n}}.
\]
Let $k=\|\beta_*\|_1$ and  $\phi(z)=z^2/2$ be the loss function. Then, we have  $|\phi''|\leq 1$ and $|\phi|\lesssim (Q+k)^2/2$ since for $\|\beta\|_1\leq Q$:
\[
    \frac{1}{2}(\beta^Tx-\beta_*^Tx)^2\leq \frac{1}{2}\|\beta-\beta_*\|_1^2 \|x\|_\infty \leq \frac{1}{2}(Q+k)^2.
\]
Then, applying \eqref{eqn: risk-bound-minimizer} to a minimizer $H\in \cH_Q$ gives
\begin{equation}\label{eqn: 199}
    \risk(H)\lesssim \log^3(n) Q^2\frac{\log(2d)}{n} + \frac{(Q+k)^2\log(1/\delta)}{n}.
\end{equation}

Now we turn to consider the linear predictor implemented by two-layer diagonal networks. 
Let $\htheta=(\ha,\hb)$ and $\hbeta=\ha\odot \hb$.
By Proposition \ref{pro: stability-1} and Theorem \ref{thm: diagnet-sharpness}, we have \wp $1-\delta$ that 
\[
\frac{2}{\eta}\geq \tr(G(\htheta))\geq (1-\varepsilon_n)\sum_{j}(\ha_j^2 + \hb_j^2)\geq 2(1-\varepsilon_n)\|\hbeta\|_1.
\]
Therefore, $f(\cdot;\htheta)\in \cH_{\hQ}$ with $\hQ\leq 1/(\eta(1-\varepsilon_n))\lesssim 1/\eta$, where the last step is because that we assume $n$ satisfies $\varepsilon_n\leq 1/2$. Plugging it into \eqref{eqn: 199} gives 
\[
\risk(\htheta)\leq \frac{(1/\eta)^2\log^2(n)\log(d)}{n} + \frac{(k+\frac{1}{\eta})^2\log(1/\delta)}{n}.
\]

\qed
\end{document}